\documentclass{article}
\usepackage{bbm,nicefrac}
\usepackage{amsmath,amssymb,amsthm}
\usepackage{bm,bbm,xfrac,url,hyperref,authblk}
\usepackage{fullpage}
\usepackage{caption,subcaption,wrapfig}

\usepackage{graphicx,xcolor,mathtools}

\newcommand{\mathsep}{,~}
\newcommand{\st}{\,\middle|\,}
\newcommand{\set}[1]{\left\lbrace #1 \right\rbrace}
\newcommand{\card}[1]{\left\lvert{#1}\right\rvert}
\newcommand{\absv}[1]{\card{#1}}
\newcommand{\norm}[2]{\left\lVert{#1}\right\rVert_{#2}}

\newcommand{\setR}{\mathbb R}
\newcommand{\setN}{\mathbb N}
\newcommand{\setS}[1]{\mathbb S^{#1}}

\newcommand{\setY}{\mathbb Y}

\newcommand{\pb}[1]{\mathbb P\left[#1\right]}
\newcommand{\expect}[1]{\mathbb E\left[#1\right]}

\DeclareMathOperator*{\argmin}{arg\,min}

\newtheorem{definition}{Definition}
\newtheorem{proposition}{Proposition}
\newtheorem{theorem}{Theorem}
\newtheorem{lemma}{Lemma}

\newcommand{\agg}{\textsc{Agg}}
\newcommand{\sign}{\textsc{Sign}}

\newcommand{\staterror}{\textsc{Err}}

\newcommand{\geomed}{\textsc{GeoMed}}
\newcommand{\clipmean}[1]{\textsc{ClipMean}_{#1}}
\newcommand{\cwmed}{\textsc{CwMed}}
\newcommand{\cwtm}{\textsc{CwTM}}

\newcommand{\sigmoid}{\sigma}

\newcommand{\calX}{\mathcal X}
\newcommand{\calN}{\mathcal N}
\newcommand{\calU}{\mathcal U}
\newcommand{\calE}{\mathcal E}
\newcommand{\calL}{\mathcal L}

\newcommand{\calY}{\mathcal Y}

\newcommand{\bH}{{\bf H}}
\newcommand{\bP}{{\bf P}}

\newcommand{\bD}{{\bf D}}
\newcommand{\bbOne}[1]{\mathbbm{1}\left[#1\right]}

\newcommand{\event}{\calE}

\newcommand{\unit}{{\bf u}}
\newcommand{\basis}{{\bf e}}

\title{The poison of dimensionality}

\author{
  Lê-Nguyên Hoang \\
  Calicarpa, Tournesol Association \\
  \texttt{len@calicarpa.com} \\
}

\begin{document}

\maketitle

\begin{abstract}
    This paper advances the understanding of how 
    the size of a machine learning model affects its vulnerability to poisoning, 
    despite state-of-the-art defenses.
    Given isotropic random honest feature vectors
    and the geometric median (or clipped mean) as the robust gradient aggregator rule,
    we essentially prove that, 
    perhaps surprisingly,
    linear and logistic regressions with $D \geq 169 H^2 / P^2$ parameters 
    are subject to \emph{arbitrary model manipulation} by poisoners, 
    where $H$ and $P$ are the numbers of honestly labeled and poisoned data points used for training.
    Our experiments go on exposing a fundamental tradeoff 
    between augmenting model expressivity 
    and increasing the poisoners' \emph{attack surface},
    on both synthetic data,
    and on MNIST \& FashionMNIST data for linear classifiers with random features.
    We also discuss potential implications 
    for source-based learning and neural nets.
\end{abstract}

\section{Introduction}

The classical theory of learning~\cite{DBLP:journals/cacm/Valiant84,geman1992neural,DBLP:conf/icml/KahaviW96}
suggests that, given $N$ training data,
learning models should have $D = \Theta(N)$ parameters.
But a vast empirical and theoretical literature on the \emph{double descent} phenomenon~\cite{DBLP:conf/iclr/ZhangBHRV17,belkin2019reconciling,DBLP:conf/isit/MuthukumarVS19,DBLP:conf/iclr/NakkiranKBYBS20,mei2022generalization,hastie2022surprises} instead suggests
that better performance could be obtained by letting $D \rightarrow \infty$.
In any case, massive data collection has led to ever larger learning models~\cite{DBLP:conf/nips/BrownMRSKDNSSAA20,DBLP:journals/jmlr/FedusZS22,DBLP:journals/jmlr/ChowdheryNDBMRBCSGSSTMRBTSPRDHPBAI23}, with now trillions, if not hundreds of trillions ($10^{14}$) of parameters~\cite{DBLP:conf/kdd/LianYZWHWSLLDLL22}.

However, these theories arguing that 
the number $D$ of parameters should grow at least linearly with the number $N$ of data 
all make two underlying assumptions: 
that all training data are ``honest'' and that they should be generalized.
In large-scale applications like content recommendation and language processing, 
this is deeply \emph{unrealistic} and \emph{ethically questionable}~\cite{DBLP:conf/icml/KallusZ18,DBLP:conf/fat/BenderGMS21}.
After all, many of these systems fit social media activity and web-crawled datasets~\cite{DBLP:conf/acl/SmithSPKCL13,DBLP:journals/jmlr/ChowdheryNDBMRBCSGSSTMRBTSPRDHPBAI23},
which are heavily \emph{poisoned} by doxed personal data, hate speech and state-sponsored propaganda~\cite{woolley2023manufacturing,facebookfakeaccounts,teamjorge}.
In fact, the survey \cite{DBLP:conf/sp/KumarNLMGCSX20} found that such \emph{data poisoning}, 
i.e. injections of misleading inputs in training datasets~\cite{DBLP:conf/icml/BiggioNL12,DBLP:conf/icml/SuyaMS0021},
has become the leading AI security concern in the industry.

Meanwhile, a growing line of research has been suggesting 
that high-dimensional training facilitates persistent poisoning attacks~\cite{hubinger2024sleeper}, 
even when state-of-the-art defenses are deployed~\cite{DBLP:journals/corr/abs-2209-15259}.
The theoretical case has mostly relied on a mathematical impossibility 
to bring the norm of the gradient at termination below $\Omega(\sqrt{D})$.
However, it is unclear whether the poisoned model is then worse than if trained with fewer parameters.
Perhaps closest to this intuition is~\cite{DBLP:conf/nips/00020F22},
who prove that the ``lethal dose'' decreases with model expressivity.

Our paper advances the understanding of how model size $D$ affects machine learning security,
given $H$ honestly labeled data and $P$ poisoned data.
Crucially, for $P = \Theta(H)$ (e.g. 1\
our results completely diverge from the common wisdom $D \geq \Theta(N) = \Theta(H)$.
In fact, in this regime, $D$ should not increase with $H$.
More precisely, we make the following contributions.

\paragraph{Contributions.}
First, when $D \geq 169 H^2 / P^2$,
we essentially prove that
using state-of-the-art poisoning defenses
(gradient descent with the geometric median or clipped mean)
actually provides \emph{zero} resilience guarantee,
at least for the two most standard learning problems (linear and logistic regression).
In fact, we prove \emph{arbitrary model manipulation} by poisoners.
    
Second, we empirically show the value of \emph{dimension reduction} under poisoning.
We expose this both on synthetic data for linear and logistic regression,
and on the MNIST \& FashionMNIST datasets given a random-feature linear classification model.
Our experiments robustly highlight a tradeoff 
between model expressivity and restricted \emph{attack surface},
when training under attack.
    
Third, we prove and leverage a property of random vector subspaces to
informally discuss the applicability of our analysis to
``sandboxed learning'' and nonlinear models.

\paragraph{Literature review.}
Poisoning attacks have recently gained a lot of attention,
see surveys~\cite{DBLP:journals/pr/BiggioR18,DBLP:conf/dsc/FanYLQX22,DBLP:conf/wcnc/WangKZH22,DBLP:journals/csur/WangMWHQR23,DBLP:journals/csur/TianCLY23,DBLP:journals/access/XiaCYM23}.
Some of the most common attacks include
label flip~\cite{DBLP:conf/icml/RosenfeldWRK20,DBLP:conf/iclr/ChenLX0023}, 
where some honest data's labels are changed, 
clean label~\cite{DBLP:conf/ccs/PapernotMGJCS17,DBLP:conf/nips/HuangGFTG20,DBLP:conf/iclr/ChenLX0023},
where the poisons are required to be classified correctly by the training classifier or a human,
and pattern backdooring~\cite{DBLP:conf/sp/WeberXKZL23}.
Our paper considers a more general class of poisoning,
where the attacker can optimize both the data and its labels.

The main classes of defenses against poisoning include
outlier removal~\cite{DBLP:journals/air/HodgeA04,DBLP:conf/sp/JagielskiOBLNL18,DBLP:conf/prdc/MullerKB20,DBLP:conf/icassp/BorgniaCFGGGGG21,DBLP:conf/kdd/ZhangCJG22,DBLP:conf/nips/ChenWW22} and
ensemble filtering~\cite{DBLP:conf/aaai/JiaCG21,DBLP:conf/iclr/0001F21,DBLP:conf/icml/00020F22,DBLP:conf/icml/Lu0Y23}.
Another defense class has relied on robust aggregation rules implemented during gradient descent training~\cite{DBLP:conf/nips/BlanchardMGS17,DBLP:journals/tsipn/YangB19,DBLP:conf/aaai/LiXCGL19,DBLP:conf/ijcai/Ma0H19,DBLP:conf/nips/SohnHCM20,DBLP:journals/corr/abs-2002-11497}.
While such defenses were introduced to secure decentralized learning,
\cite{DBLP:conf/icml/FarhadkhaniGHV22} proved a deep connection with data poisoning.

Our results contribute to a growing body of work
on the theoretical impossibility of securing learning~\cite{DBLP:conf/icml/MhamdiGR18,DBLP:conf/ccs/JagielskiSHO21,DBLP:conf/nips/El-MhamdiFGGHR21,DBLP:conf/nips/FangL0D0022,DBLP:conf/focs/GoldwasserKVZ22,DBLP:conf/iclr/KarimireddyHJ22,DBLP:journals/corr/abs-2309-13591}.
As summarized by~\cite{DBLP:journals/corr/abs-2209-15259}, 
the impossibility especially arises under data heterogeneity and in high dimension.
However, to the best of our knowledge, 
we provide the first \emph{arbitrary model manipulation} result
despite the use of a state-of-the-art defense.

\paragraph{Structure.}
In the sequel, Section~\ref{sec:setting} precisely defines the setting of our analysis.
Section~\ref{sec:theorem} introduces our main theorem and the key intuition underlying it.
Section~\ref{sec:experiments} presents further incriminating experimental results.
Section~\ref{sec:discussions} informally discusses generalizations.
Section~\ref{sec:conclusion} concludes.

\section{Setting}
\label{sec:setting}

Given a list $\bD 
\triangleq \set{
    (x_1, y_1), (x_2, y_2), \ldots, (x_N, y_N)
} 
\in \left(\setR^{D} \times \setY \right)^N$ of $N$ training feature-label pairs, 
also known as the \emph{training dataset},
we learn a linear prediction $\alpha^T x$ of labels $y$.
The error on a feature-label pair $(x,y)$ is given by the prediction error,
which we write $\ell(\alpha | x, y)$.

We assume that (honest) feature vectors are independent and identically distributed,
as $x_n \sim \calX$,
while (honest) labels are only dependent on a ``ground truth'' conditional distribution $\calY(x)$.
In the theoretical section of this paper,
we focus on the two most standard machine learning models, linear regression and logistic regression,
where $\calY$ is parameterized by a vector $\beta$.

\paragraph{Linear regression (without noise).} 
In least square linear regression\footnote{Our main theorem still holds with noise, see Appendix.}, 
it is common to assume a hidden true vector $\beta \in \setR^D$
such that $\setY \triangleq \setR$ and $\calY_{\beta}(x)$ is the Dirac distribution reporting $\beta^T x$.
The loss is given by $\ell(\alpha | x, y) \triangleq \frac{1}{2} (\alpha^T x - y)^2$.
Its gradient is $\nabla \ell(\alpha | x, y) = (\alpha^T x - y) x$.

\paragraph{Logistic regression.}
Logistic regression also depends on a hidden vector $\beta \in \setR^D$.
But now $\setY \triangleq \set{0, 1}$ and $\calY_{\beta} (x)$ is a Bernoulli distribution with
$\pb{y = 1 | \beta, x} = \sigmoid(\beta^T x)$,
where $\sigmoid$ is the sigmoid function $\sigmoid(t) = (1+e^{-t})^{-1}$.
The cross-entropy loss on an input $(x, y)$ is 
$\ell(\alpha | x, y) \triangleq - \ln \sigmoid(\alpha^T x)$ if $y = 1$
and $\ell(\alpha | x, y) \triangleq - \ln (1 - \sigmoid(\alpha^T x))$ if $y = 0$.
The gradient is 
$\nabla \ell(\alpha | x, y) = (\sigmoid(\alpha^T x) - y) x$.

\paragraph{}
Assuming that all data are correctly labeled, 
a common solution to the learning problem is to minimize an empirical error.
Considering the average empirical loss yields:
\begin{equation}
\label{eq:loss}
\calL (\alpha | \bD)
= \frac{1}{N} \sum_{n \in [N]} \ell(\alpha | x_n, y_n).
\end{equation}
While it is not uncommon to add a regularization,
several works on ``ridgless regression''~\cite{DBLP:journals/simods/BelkinHX20} 
suggested optimal performances with no regularization (assuming honest data only).

For linear regression, in the regime $N \geq D$,
i.e. there are more data than parameters,
it is well-known that, with probability 1, 
$\calL$ is strongly convex, and thus has a unique minimum.
In the case of logistic regression, 
further assumptions are needed to guarantee the existence of a minimum;
but it can be guaranteed with high probability for $N = \Omega(D)$.
In both cases, for a fixed $D$ and in the limit $N \rightarrow \infty$,
assuming the data $(x_n, y_n)$ to be independent with $x_n \sim \calN(0, I_D)$ and $y_n \sim \calY_{\beta} (x)$,
the learned model essentially converges to the true model, see e.g.~\cite{DBLP:conf/icml/FarhadkhaniGHV22}.

\subsection{Poisoning attacks}

In the sequel, we consider an adversarial setting where not all data are correctly labeled.
More precisely, the training dataset is assumed to result from the merging of two sets 
$\bH \triangleq \set{(x_h, y_h)}_{h \in [H]}$ 
and $\bP \triangleq \set{(x_{H+p}, y_{H+p})}_{p \in [P]}$,
which respectively contain \emph{honest} $(x_h, y_h)$
and \emph{poisonous} $(x_{H+p},y_{H+p})$ data.
Their respective cardinalities are $H$ and $P$, with $H+P = N$.
Like previously, the honest data are assumed to be i.i.d., 
with, for all $h \in [H]$, $x_h \sim \calX$ and $y_h \sim \calY(x_h)$.
However, each poisonous data may be arbitrarily chosen in $\setR^D \times \setY$, potentially by an adversary.

Evidently, the difficulty arises when the training set $\bD$ does not distinguish honest from poisoning data\footnote{
Formally, we may consider that there is a random permutation $s : [N] \rightarrow [N]$
such that the $n$-th observed element $(x_n^\bD, y_n^\bD)$ in the training dataset $\bD$
is $(x_n^\bD, y_n^\bD) \triangleq (x_{\sigma(n)}, y_{\sigma(n)})$.
By considering permutation-invariant learning rules,
the training then remains the same no matter which random permutation is selected.
}.
To account for the ignorance of the learning model,
we demand that our training be invariant to shuffles of data indices $n \in [N]$.
Note that minimizing $\calL$ indeed does so.

In this setting, as long as $P \geq 1$, 
learning by minimizing a regularized empirical loss is highly insecure.
To prove this, we 
leverage the well-known gradient inversion lemma.

\begin{lemma}[Gradient inversion]
\label{lemma:gradient_inversion}
Consider linear or logistic regression.
For any $\alpha \in \setR^D$ and $g \in \setR^D$,
there exists $(x, y) \in \setR^D \times \setY$ 
such that $\nabla \ell (\alpha | x, y) = g$.
\end{lemma}

\begin{proof}
Define $(g, \alpha^T g - 1)$ for linear regression,
and $(\frac{g}{\sigma(\alpha^T g)}, 0)$ for logistic regression.
\end{proof}

It then follows that, 
for all learning dimensions $D$, 
linear and logistic regressions without defenses
are vulnerable to \emph{arbitrary model manipulation}
by a single poisoner.

\begin{proposition}[Arbitrary model manipulation by a single poisoner]
\label{prop:manipulability}
Consider linear or logistic regression.
For any model size $D \in \setN$, any target vector $\alpha \in \setR^D$ and 
any subset $\bH \triangleq \set{(x_h, y_h)}_{h \in [H]}$ of $H$ data,
there is a data $\bP \triangleq \set{(x_{H+1}, y_{H+1})}$
such that $\alpha$ minimizes $\calL(\cdot | \bH \cup \bP)$.
\end{proposition}

\begin{proof}
Let $g \triangleq - \nabla \calL(\alpha | \bH)$.
By Lemma~\ref{lemma:gradient_inversion}, 
we know that there exists $(x_{H+1}, y_{H+1})$ such that
$\nabla \ell (\alpha | x_{H+1}, y_{H+1}) = g$.
We then have $\nabla \calL(\alpha | \bH \cup \bP) = \nabla \calL(\alpha | \bH) + g = 0$,
which proves that $\alpha$ minimizes the convex loss $\calL$.
\end{proof}

In practice, there may be restrictions on the poisoning feature vector $x$.
Typically, images $x$ may be constrained to belong to\footnote{Or be drawn from a specific distribution to which legitimate data are assumed to belong.} $([0, 255] \cap \setN)^D \subset \setR^D$.
In particular, bounding the set of allowed feature vectors $x$ then prevents gradient inversion,
and thus \emph{arbitrary model manipulation} by a single poison.
Note however, that the key intuition of our main theorem (Section~\ref{sec:intuition}) nevertheless applies,
as it leverages the direction of poison-based gradient (not its size).
This suggests that arbitrary model manipulation still holds for $D \geq \Omega(H^2/P^2)$,
even given feature space constraints.

\subsection{Securing learning with the geometric median or clipped mean}
\label{sec:robustified_learning}

To secure training against poisoning, 
we propose to aggregate the gradients coming from each data 
using so-called \emph{Byzantine-resilient gradient aggregation rules}~\cite{DBLP:conf/nips/BlanchardMGS17,DBLP:conf/icml/MhamdiGR18}.
To understand, first observe that the gradient of $\calL$ can be written
\begin{equation}
\label{eq:average_gradient}
\nabla \calL(\alpha | \bD) 
= \frac{1}{N} \sum_{n \in [N]} \nabla \ell(\alpha | x_n, y_n),
\end{equation}
which is subject to \emph{arbitrary gradient manipulation} by a single poisonous gradient.
Given the gradient inversion lemma (Lemma~\ref{lemma:gradient_inversion}),
this implies \emph{arbitrary model manipulation} (Proposition~\ref{prop:manipulability}).

To secure the training, 
\cite{DBLP:conf/nips/BlanchardMGS17} replaces the average
by a robust mean estimator.
In the theoretical part of our paper, 
we focus on the $\Delta$-\emph{clipped mean}~\cite{DBLP:conf/iclr/KarimireddyHJ22,DBLP:conf/nips/KoloskovaMCRM23}
and the \emph{geometric median}~\cite{lopuhaa1989relation,minsker2015geometric,DBLP:conf/iclr/KarimireddyHJ22,DBLP:conf/aistats/AcharyaH0SDT22,DBLP:conf/aistats/El-MhamdiFGH23}.
Recall their definitions.

\begin{definition}
\label{def:clipped_mean}
For any $N \in \setN$, and given any vectors $z_1, \ldots, z_N \in \setR^D$,
the $\Delta$-clipped mean is given clipping the vectors whose norms exceed $\Delta$,
and by averaging the clipped vectors, i.e.
\begin{equation}
\clipmean{\Delta} (z_1, \ldots, z_N) 
\triangleq 
\frac{1}{N}
\sum_{n \in [N]} \min \set{1, \frac{\Delta}{\norm{z_n}{2}}} z_n.
\end{equation}
\end{definition}

\begin{definition}
\label{def:geometric_median}
For any $N \in \setN$, and given any vectors $z_1, \ldots, z_N \in \setR^D$,
the set $\geomed(z_1, \ldots, z_N)$ of geometric medians 
is defined as the minimum of the sum of distances to the $z_n$'s, i.e.
\begin{equation}
\geomed(z_1, \ldots, z_N) 
\triangleq 
\argmin_{g \in \setR^d} 
\sum_{n \in [N]} \norm{g - z_n}{2}.
\end{equation}
\end{definition}

Whenever the inputs $z_1, \ldots, z_N$ are not along a line (or if $N$ is odd),
the geometric median is known to be unique.
In our setting, because the vectors $x_n$ have a continuous probability distribution over $\setR^d$,
and since gradients are colinear with $x_n$,
with probability 1, the geometric median is uniquely defined.
$\geomed (z_1, \ldots, z_N)$ is then invariant up to index labeling,
as demanded to model the prior ignorance about which data is poisonous.

Following \cite{DBLP:conf/nips/BlanchardMGS17},
to increase poisoning resilience,
we modify the standard gradient descent algorithm,
by replacing \eqref{eq:average_gradient} with the following robustified gradient:
\begin{equation}
\widehat{\nabla \calL} (\alpha | \bD) 
\triangleq \agg \left( \nabla \ell(\alpha | x_1, y_1), \ldots, \nabla \ell(\alpha | x_N, y_N) \right).
\end{equation}
Assuming $H > P$, both $\agg = \clipmean{\Delta}$ and $\agg = \geomed{}$ are well-known 
to be resilient to \emph{arbitrary gradient manipulation}~\cite{lopuhaa1989relation, DBLP:conf/aistats/El-MhamdiFGH23}.
But does this imply that gradient descent with such aggregation rules also prevents \emph{arbitrary model manipulation}?
To formally answer this question, 
let us define stationary points.

\begin{definition}
We say that $\alpha \in \setR^D$ is a $(\agg, \bD)$-stationary point
if $\widehat{\nabla \calL} (\alpha | \bD) = 0$.
\end{definition}

Stationary points are of special interest as, 
if the gradient descent with \agg{} ever reaches such a point, 
then it will be stuck at this point.

\section{Main result}
\label{sec:theorem}

Our main result is essentially that, when $D \geq 169 H^2 / P^2$,
even though they \emph{cannot} arbitrarily manipulate the robustified gradient,
poisoners can still arbitrarily manipulate the trained model.
\emph{Arbitrary model manipulation} does \emph{not} require \emph{arbitrary gradient manipulation}.

\begin{theorem}
\label{th:negative}
Consider linear or logistic regression, 
and $\agg \in \set{ \clipmean{\Delta}, \geomed }$ (for any $\Delta \geq 0$). 
Suppose the honest feature vector distribution $\calX$ is isotropic, 
i.e. invariant under orthogonal transformations.
Assume $D \geq 1024$. $H \geq 6272 D \ln D$ and $D \geq 169 H^2 / P^2$.
Take any target model $\alpha \in \setR^D$.
Then, with probability at least $1-\nicefrac{2}{D}$ over the random honest set $\bH$ of cardinal $H$,
there exists a poisonous set $\bP$ of cardinal $P$
such that $\alpha$ is a $(\agg, \bH \cup \bP)$-stationary point. 
\end{theorem}

\subsection{Five remarks on the main theorem}
\label{sec:remarks}

Before discussing the proof of Theorem~\ref{th:negative}, 
we make five remarks on its limits and implications.

\paragraph{High probability.}
Our theorem should have been concerning 
even if the probability of vulnerability was barely positive.
However, for large $D$, 
we prove that vulnerability is essentially a guarantee.

\paragraph{The overparameterized regime.}
It is straightforward to modify the final step of our proof (Appendix~\ref{app:final}),
to derive \emph{arbitrary model manipulation} even when $H \leq 6272 D \ln D$, 
assuming $P \geq 728 \sqrt{H \ln D}$.
In fact, informally, in the overparameterized regime $D \geq \tilde\Omega(H)$, 
then the number $P$ of poisons to arbitrarily manipulate the model need only be sublinear in $H$.

\paragraph{Independence to signal-to-noise ratio.}
Perhaps surprisingly, 
our theorem does not depend on the noise in honest data.
In fact, since the result only follows from the isotropy of random honest feature vectors $x$
and the fact that $\nabla \ell(\cdot | x, y)$ is colinear with $x$,
it is \emph{independent} from how honest data is labeled.
Even for labeling completely different from what we assumed,
our theorem holds.

\paragraph{The attack.}
Especially in the case of the geometric median, the poisoning is remarkably easy to implement. 
Namely, poisoners merely need to provide data $(x_{H+p}, y_{H+p})$ for which $\nabla \ell(\alpha|\pi_d(x_{H+p}), y_{H+p})) = 0$.
For linear regression, this can be done by drawing $x_{H+p} \sim \calN(0, I_D)$,
and by adding the label $y_{H+p} \triangleq \alpha^T x_{H+p}$,
thereby making poisonous feature vectors statistically indistinguishable from honest feature vectors.

\paragraph{Quantitative implication.}
Assuming $H/P = 100$ (i.e. roughly only 1\
secure training demands $D \leq 1.7 \cdot 10^6$.
Thus linear models that have over millions of parameters
and whose training involves web-crawled data 
must be considered \emph{highly insecure}.

\begin{wrapfigure}{r}{0.5\linewidth}
    \vspace{-30pt}
    \includegraphics[width=\linewidth]{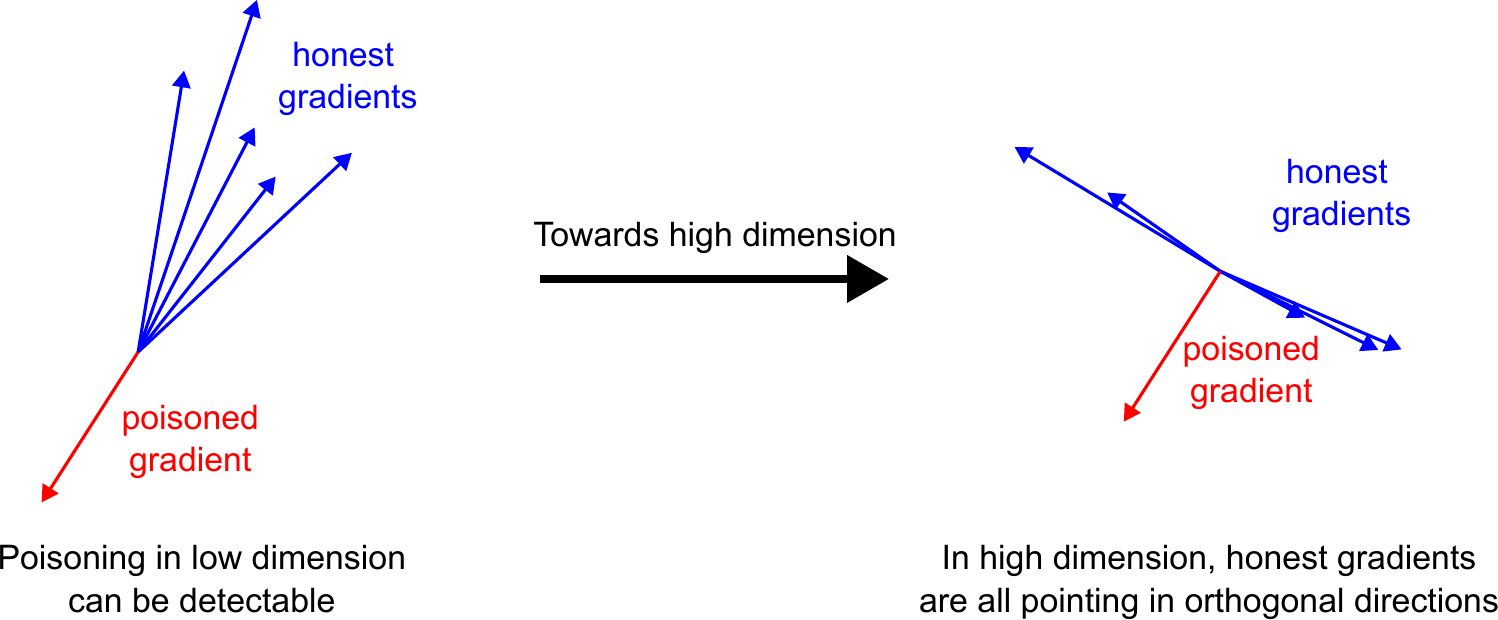}
    \caption{In high dimension, 
    correct gradients fail to point in the right direction,
    which makes poisoning vastly more devastating,.}
    \label{fig:poisoning}
    \vspace{-10pt}
\end{wrapfigure}

\subsection{Key intuition}
\label{sec:intuition}

While our proof of Theorem~\ref{th:negative} is specific to $\clipmean{\Delta}$ and \geomed{}-gradient descent,
it relies on a remarkably simple insight,
which might suggest a more fundamental learning impossibility under poisoning in high dimension.

Namely, the key intuition is that, 
if the gradient of the loss for a random honest data is isotropically distributed in $\setR^D$,
then, with high probability, 
it will be almost orthogonal to the right learning direction (towards $\beta$),
as well as to any other honest gradient (see Figure~\ref{fig:poisoning}).
This makes each honest data poorly informative.

In contrast, 
the poisoners may select feature vectors 
whose direction is fully aligned with their preferred update direction.
Thereby, intuitively, 
\emph{each poison can be made $\sqrt{D}$ times more disinformative than how informative each honest data is}.
This suggests that, when $P \sqrt{D} \gg H$, i.e. $D \geq \Theta(H^2 / P^2)$,
defending against poisoned data might be hopeless,
perhaps regardless of which (reasonable) gradient aggregation rule is used.

Note that our case strongly relies on the full control of attacks on the poisonous feature vectors.
In particular, this suggests that, in high dimension, 
label-flipping attacks are $\sqrt{D}$ times less harmful,
and should thus not be regarded as a gold standard of poisoning.
Typically, artists may leverage AIs to optimize images (feature) with misleading captions (label)
to fool generative algorithms~\cite{DBLP:conf/uss/ShanCW0HZ23}
while online propaganda may both generate (feature) and like (label) social media content to manipulate recommendation algorithms~\cite{DBLP:journals/eswa/ChenBQYLS24}.

\subsection{Proof sketch}

We now provide a brief sketch of our four-step proof.
The full proof is detailed in the appendix.

First, in Appendix~\ref{app:geomed},
using the gradient inversion lemma, 
we observe that $\alpha$ can be made $(\geomed{}, \bH \cup \bP)$-stationary,
if the sum $\sum_{h \in [H]} \frac{g_h}{\norm{g_h}{2}}$ of normalized honest gradients is at most $P$.
Indeed, 
the poisoners may then simply report data $(x_{H+p}, y_{H+p})$ 
such that $\nabla \ell(\alpha|\pi_d(x_{H+p}), y_{H+p})) = 0$.
A similar property holds for clipped mean 
(in this case, the sum of clipped honest gradients needs to be at most $\Delta P$).

Second, because the honest gradients $g_h$ are colinear with the feature vectors $x_h$,
which are isotropically distributed,
especially in high dimension,
$g_h$ is unlikely to point towards its expected value.
As a result, the expectation of normalized (or clipped) gradients, 
which is an average of ``misguided'' unit/clipped vectors,
cannot be large, especially in high dimension.
In fact, in Appendix~\ref{app:expectation}, 
we prove the norm to be at most $12 / \sqrt{D}$ for normalized gradients,
and $12 \Delta / \sqrt{D}$ for $\Delta$-clipped gradients.

Third, in Appendix~\ref{app:concentation}, we use concentration bounds to guarantee that
the empirical sum of normalized/clipped gradients is roughly its expectation.
Deriving explicit constants for these bounds is the most technical part of the proof.
Isotropy is leveraged to prove that,
along directions orthogonal to its expected value,
normalized/clipped gradients are sub-Gaussian with a parameter $O(1/\sqrt{D})$.

Combining it all yields the main theorem (Appendix~\ref{app:final}).

\section{Experiments on dimension reduction}
\label{sec:experiments}

Clearly, Theorem~\ref{th:negative} calls for \emph{dimension reduction} to secure model training.
In this section, we formalize this strategy,
and empirically study the impact of the learning dimension $d$ 
on robust training under poisoning, on both synthetic data and on MNIST.

\subsection{Synthetic data experiments}
\label{sec:synthetic_data}

We first consider synthetic data,
which provides insights into the value of dimension reduction.

\paragraph{Dimension reduction: Limit the attack surface}
\label{sec:dimension_reduction}

Consider a map $\pi_d : \setR^D \rightarrow \setR^d$.
Dimension reduction merely corresponds to training from $(\pi_d(x_n), y_n)$, instead of $(x_n, y_n)$.
Essentially, we force our algorithms to be blind to some inputs;
perhaps to their overwhelming majority.
In this section, $\pi_d$ is simply defined as the orthogonal projection 
on the first $d$ coordinates.

Now, intuitively, this may appear to be bizarre.
How can removing information increase security?
Our key insight is that this literally reduces the \emph{attack surface}\footnote{
``Attack surface'' is a terminology widely used in classical cybersecurity.
In our setting, amusingly, it takes a literal mathematical meaning,
as it describes the subspace of attack feature vectors that the adversary can exploit.
}, 
i.e. the combinatorial space that the adversary can exploit.
In particular, the adversary can only attack with $P$ vectors of dimension $d$,
instead of $P$ vectors of dimension $D$.
But which value of $d$ should be selected?
To address this question, we consider the following experimental setup.

\paragraph{True model.}
To simulate the fact that the system designer may successfully prioritizes 
the dimensions with the largest signal,
we consider a true model of the form $\beta \triangleq (1^{-\omega}, 2^{-\omega}, \ldots, D^{-\omega})$, for $\omega \geq 0$.
Our simulations are run with $\omega \triangleq 0.5$.

\paragraph{Statistical error.}
Our (poisoned) trained model is evaluated by the statistical error:
\begin{equation}
    \staterror(\alpha|\calY) \triangleq 
    \expect{ \ell(\alpha | x, y) \st x \sim \calN(0, I_D), y \sim \calY(x) }. 
\end{equation}
As shown in the Appendix, for dimension-reduced linear regression, 
this equals
$\norm{\beta}{2}^2 + \norm{\alpha - \pi_d(\beta)}{2}^2$.

\paragraph{Poisoning defense.}
For more completeness, 
instead of the geometric median\footnote{
The error of state-of-the-art algorithms~\cite{DBLP:conf/stoc/CohenLMPS16, DBLP:journals/tsp/PillutlaKH22} grow proportionally to the sum-of-distance loss,
which can itself be made arbitrarily large by vectors.
Adapting them to force an attack-independent error unfortunately allows attackers
to arbitrarily slow their runtime.
}
and of clipped mean,
our experiments consider two other state-of-the-art robust mean estimators,
namely the \emph{coordinate-wise median} (\cwmed{}\footnote{
The output vector is, for each coordinate, the median of the inputs' values on this coordinate.
})
and the \emph{$P$-trimmed mean} (\cwtm{}\footnote{
For each coordinate, the $P$ most extreme inputs are discarded.
The remaining values are averaged).
}~\cite{DBLP:conf/icml/YinCRB18,DBLP:conf/nips/El-MhamdiFGGHR21},
for which \emph{arbitrary model manipulation} does not hold in general\footnote{
This is especially the case for $\beta$ colinear with a basis vector,
as the problem is then reduced to one-dimensional training.
}.

\paragraph{Poisoning attack.}
By virtue of the gradient inversion lemma, without loss of generality,
like~\cite{DBLP:conf/icml/FarhadkhaniGHV22},
we consider gradient attacks instead of data poisoning.
We focus on the ``antimodel gradient attack'',
which is closely related to gradient ascent of the statistical error.
Given parameters $\alpha$, the right training direction points from $\alpha$ to $\pi_d(\beta)$.
This prompts us to define poisoned gradients by $g^{(P)} \triangleq \Lambda (\pi_d(\beta) - \alpha)$,
with a large multiplier $\Lambda \geq 0$.
Our experiments are run with $\Lambda \triangleq 10^3$.

\paragraph{Hyperparameters.}
We consider $H \triangleq 5000$ honest data, with $\calX \triangleq \calN(0, I_D)$.
We set $D \triangleq 5000$ for linear regression, 
and $D \triangleq 1000$ for logistic regression,
to generate honest data.

\paragraph{Reproducibility.}
Our experiments are repeated using 10 seeds (see Supplementary Material),
which define the random honest training dataset $\bH$.
The results are averaged. 
We also plot the variance between runs with different seeds.

\begin{figure*}[ht]
    \centering
    \begin{subfigure}{0.245\textwidth}
        \includegraphics[width=\linewidth]{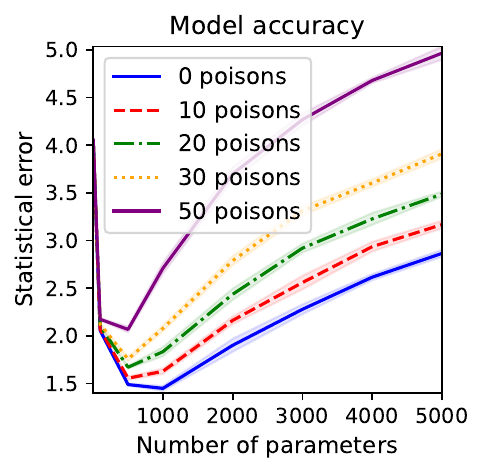} 
        \caption{\tiny Linear regression with \cwmed{}}
    \end{subfigure}
    \begin{subfigure}{0.235\textwidth}
        \includegraphics[width=\linewidth]{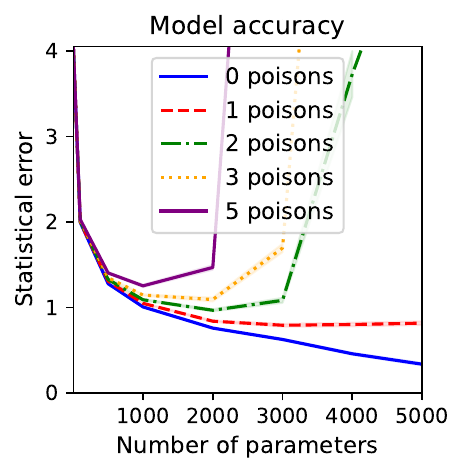} 
        \caption{\tiny Linear regression with \cwtm{}}
    \end{subfigure}
    \begin{subfigure}{0.248\textwidth}
        \includegraphics[width=\linewidth]{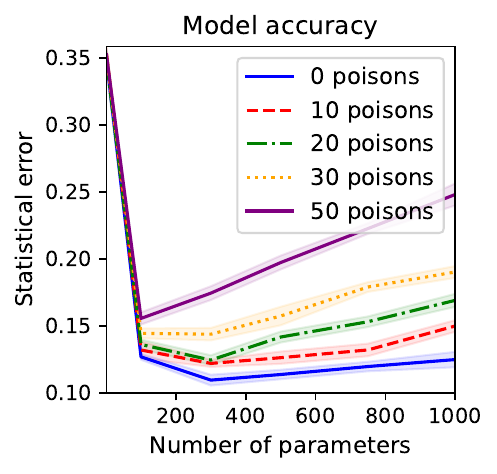} 
        \caption{\tiny Logistic regression with \cwmed{}}
    \end{subfigure}
    \begin{subfigure}{0.24\textwidth}
        \includegraphics[width=\linewidth]{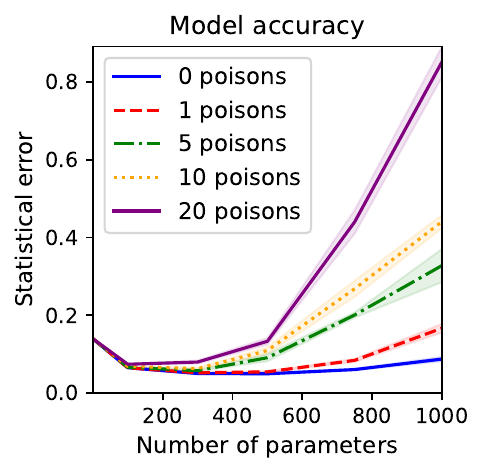}
        \caption{\tiny Logistic regression with \cwtm{}}
    \end{subfigure}
    \caption{Statistical errors under poisoning,
    for varying model sizes $d$ and number $P$ of poisons.}
    \label{fig:results}
\end{figure*}

\paragraph{Results}
The statistical errors of our trained models are depicted in Figure~\ref{fig:results}.
In all cases,
we observe an initial gain of increasing the dimension,
followed by a harm that eventually outweighs the no-training error\footnote{
$\alpha = 0$ implies a statistical error of $\norm{\beta}{2}^2$ for linear regression,
and of $\ln 2 \approx 0.693$ for logistic regression.
} ($\alpha = 0$).
In particular, the plots suggest that \emph{arbitrary model manipulation} may often hold
beyond the particular case of \geomed{}-based and $\clipmean{}$-based defense.

Note that, in our experiments, 
logistic regression suffers from overfitting data, even in the absence of poisons.
Moreover, the statistical errors do not vanish without poisons under coordinate-wise median defense.
This is not unexpected, as the coordinate-wise median diverges from the mean.
Overall, only the case of linear regression with trimmed mean has a vanishing statistical error as $d = D$.

Perhaps surprisingly, training is extremely vulnerable to small amounts of poisoning.
In particular, the $P$-trimmed mean for linear regression is deeply harmed 
by merely $2$ poisons out of $5,002$ data.
Even the visibly more resilient coordinate-wise median is deeply harmed by less than 1\% of poisoning inputs,
despite still relatively modest values of $d$, 
like $d = 1000$ in the context of logistic regression.
Our experiments thus expose the fact that 
the ``safe'' regime $d \ll H^2 / P^2$ is actually far from safe.

Intriguingly, the optimal learning dimension $d$ is observed to be a decreasing function 
of the number $P$ of poisoned data.
In fact, the experiments invites us to conjecture 
that the optimal learning dimension $d$ is closely connected to the ratio $H/P$.
Note that it evidently also depends on $\omega$.
Larger values of $\omega$ likely decrease the optimal value of $d$,
as larger coordinates quickly become highly uninformative.

\begin{wrapfigure}{r}{0.5\linewidth}
    \vspace{-40pt}
    \includegraphics[width=0.24\textwidth]{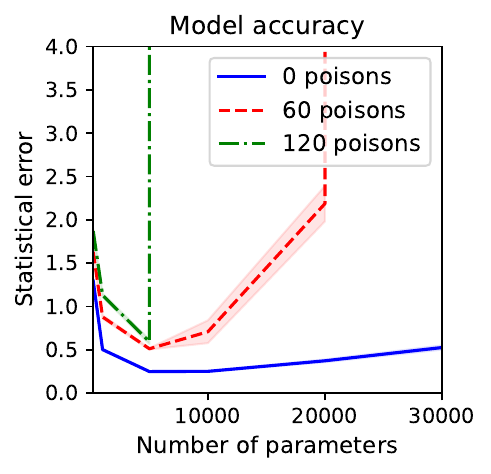} 
    \includegraphics[width=0.24\textwidth]{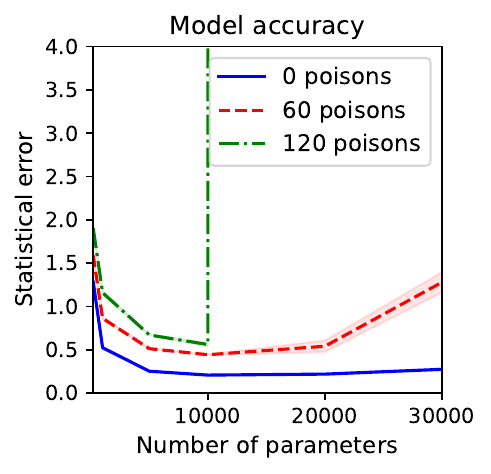} 
    \caption{Cross-entropy of a random feature linear classifier on MNIST (left) and FashionMNIST (right), trained using gradient descent with trimmed mean on the training sets, and evaluated on the validation sets, as a function of the number of parameters of these models.}
    \label{fig:results2}
    \vspace{-30pt}
\end{wrapfigure}

\subsection{Experiments on standard datasets}
\label{sec:mnist_experiments}

In this subsection, we further evaluate the impact of high-dimensional training
on the performance of robustified models under poisoning attack,
but now for the MNIST dataset~\cite{lecun1998mnist}
and the FashionMNIST dataset~\cite{DBLP:journals/corr/abs-1708-07747},
which both consist of feature-label pairs $(x_n, y_n) \in \setR^{784} \times \set{0, 1, \ldots, 9}$.

\paragraph{Random feature linear classifier.}
We consider a linear classifier modle based on $d$ random features,
in the spirit of e.g.~\cite{DBLP:conf/nips/RahimiR07,mei2022generalization}.
More precisely, we draw $d$ random vectors $w_1, \ldots, w_d \sim \calN(0, I_{784})$.
Each feature $x_n \in \setR^{784}$ is then replaced by the vector
$\tilde x_n = (\sigma(w_1^T x_n), \ldots, \sigma(w_d^T x_n)) \in \setR^d$,
where $\sigma$ is an activation function.
Here, we consider the sigmoid function.
A linear classifier is then trained on the pairs $(\tilde x_n, y_n) \in \setR^d \times \set{0, 1, \ldots, 9}$
(which executes softmax on linear forms of $\tilde x_n$).
The number of parameters of this model is $10 d$.
Note that the classifier can be regarded as a two-layer fully connected neural net
with a random fixed first layer.

\paragraph{Attack.} 
In this section, we again consider gradient (ascent) attacks instead of data poisoning. 
We also assume that, during training, each batch of size $b = 1000$ is contaminated by $\lfloor P / H \rfloor$ poisons.
Given that the MNIST and FashionMNIST training sets each have 60,000 entries,
the ``60 poisons'' setting corresponds to 1 poisonous input per batch.

\paragraph{Results.}
The performance is measure with cross-entropy on the MNIST and FashionMNIST validation sets.
Again we observe U-shaped curves, 
with very poor performance for high-dimensional models,
despite an extremely small amount of poisoning ($P/H \leq 1\%$).
Additional details about the experiments are provided in Appendix~\ref{app:experiments}.

\section{Informal discussions on more general cases}
\label{sec:discussions}

Let us now sketch informal arguments 
to estimate the generalizability of our theoretical and empirical results,
thereby providing further insights into high-dimensional poisoning.
Our subsequent discussions are based on the following bound,
which intuitively says that the informativeness of a sampled gradient
projected on a $d$-dimensional random subspace of $\setR^D$
grows as $\tilde\Theta(\sqrt{d / D})$.

\begin{proposition}    
\label{prop:subspace}
    Let $V$ be an isotropically random $d$-dimensional subspace of $\setR^D$.
    Suppose $v$ is a random unit vector that must be in $V$.
    Then $\norm{\expect{v}}{2} \leq 1150 \sqrt{\frac{d \ln (1+d)}{D}}$.
\end{proposition}

\begin{proof}[Proof sketch]
    Denote $u = \frac{\expect{v}}{\norm{\expect{v}}{2}}$ the direction of the expectation of $v$.
    $V$ can be constructed as the subspace spanned by $d$ i.i.d normal vectors $w_1, \ldots, w_d$,
    which form a quasi-orthonormal basis of $V$ (up to scaling).
    Formalizing this (with concentration bounds) allows to bound the coordinates of $v \in V$ in this basis, given that $v$ is unitary.
    Now, $u^T v$ is a sum of terms $u^T w_1, \ldots, u^T w_d$,
    weighted by such coordinates.
    Using concentration bounds on terms $u^T w_i$ then guarantees that $u^T v$ is unlikely to be large.
    We conclude by taking a union bound over a partition into extreme values of $\absv{u^T v}$.
    The full proof is in Appendix~\ref{app:generalizations}.
\end{proof}

\subsection{Source-based learning}
\label{sec:user_resilience}

To increase security, it is useful to partition data based on their sources~\cite{DBLP:conf/icml/FarhadkhaniGHV22,DBLP:journals/corr/abs-2209-15259},
i.e. to write $\bD = \bigcup_{s \in [S]} \bD_s$,
where $[S] \triangleq \set{1, 2, \ldots, S}$ is now a set of data sources
and where $\bD_s \triangleq \set{(x_n, y_n)}_{n \in [N_s]}$ is the data from source $s$,
which is of cardinal $N_s$.
The loss for source $s$ can then be written 
$\ell(\alpha | \bD_s) \triangleq \frac{1}{N_s} \sum_{s \in [N_s]} \ell(\alpha | x_k, y_k)$.
This may be called \emph{source-based learning}, 
as it puts each source's data in a separate sandbox, rather than in a common pool.

Now, it is the set $[S]$ of sources that can be partitioned 
into two subsets of cardinalities $H$ and $P$,
where the former is a subset of honest sources, 
and the latter is the subset of poisoning sources.
We may then use a robust aggregation \agg{} to combine the gradients from the different sources, as
\begin{equation}
\widehat{\nabla \calL} (\alpha | \bD) 
\triangleq \agg{} \left( \nabla \ell(\alpha | \bD_1), \ldots, \nabla \ell(\alpha | \bD_S) \right).
\end{equation}
Assuming that data of honest users are labeled using $\calY_{\beta}$,
sandboxed learning with a robust aggregator will increase security.
First, the security no longer depends on the amount of poisonous data;
it will rather depend on the fraction of poisoner users accepted in the system.
But there is more. 

Because the gradient $\nabla \ell(\alpha | \bD_h)$ of an honest source $h \in [H]$ cumulates $N_h$ data,
and will thus be more informative.
In fact, assuming the feature vectors labeled by source $h$ are i.i.d. and isotropic,
we know that $\nabla \ell(\alpha | \bD_h)$ is an isotropic random $N_h$-dimensional subspace.
Thus Proposition~\ref{prop:subspace} applies,
and suggests that source $h$'s data are $\tilde \Theta(\sqrt{N_h/D})$ times less informative than an optimized poison.

Following the key intuition (Section~\ref{sec:intuition}),
it seems that \emph{arbitrary model manipulation} would arise
if $P \geq \tilde \Omega\left( \frac{\sum_h \sqrt{N_h}}{\sqrt{D}} \right)$.
This is equivalent to $D \geq \tilde \Omega \left( \expect{\sqrt{N_S}}^2 H^2 / P^2 \right)$,
where $N_S$ is the number of data provided by a random honest source.
In particular, if all honest sources $h \in [H]$ report $N_h = N_S$ honest data,
then this corresponds to $D \geq \tilde \Omega(N_S H^2 / P^2)$.

Unfortunately, we fall short of a proof,
because the concentration analysis (Appendix~\ref{app:concentation}) does not straightforwardly generalize.
We leave this conjecture open.

Besides, an additional difficulty arises in practice.
Namely, in most applications, especially language models and recommendation AIs,
different sources $h$ may use different distributions $\calY_{\beta_h}$.
This is what~\cite{DBLP:journals/corr/abs-2209-15259} refers to as \emph{fundamental heterogeneity}.
Yet ~\cite{DBLP:conf/icml/FarhadkhaniGHV22} proved that 
$\beta_h \neq \beta_k$ for different honest users $h \neq k$ 
further increases the vulnerability to poisoners.

\subsection{Neural networks}
\label{sec:neural_nets}

Assume $y$ is predicted by $f_\alpha(x)$,where $f_\alpha$ is a neural network, with $\alpha \in \setR^D$.
Theorem~\ref{th:negative} no longer applies.
Nevertheless, the key intuition of Section~\ref{sec:intuition} suggests the following.

\paragraph{Regression.}
Suppose that $\setY \triangleq \setR^{d_y}$
and that the loss on input $(x, y)$ can be written
$\ell(f_\alpha | x, y) = \frac{1}{2} \norm{f_{\alpha} (x) - y}{2}^2$.
We then have
$\nabla_\alpha \ell = \left( \nabla_\alpha f_{\alpha} (x) \right) ( f_{\alpha} (x) - y )$,
which must belong to the image of matrix $\nabla_\alpha f_{\alpha} (x) \in \setR^{D \times d_y}$.
Thus the gradient is constrained to lie in
a $d_y$-dimensional random vector subspace.
The intuition of Theorem~\ref{th:negative} might then apply.
As a result, 
we might have to fear \emph{arbitrary model manipulation} in the regime $D \geq \tilde \Omega(d_y H^2 / P^2)$.

\paragraph{Classification.}
Suppose $\setY$ is a finite set.
Given $x$, the gradient on data $(x, y)$ must be 
one of $\card{\setY}$ possible values $\nabla_\alpha \ell (\alpha | x, y)$,
for $y \in \setY$,
which span a random $\card{\setY}$-dimensional subspace.
As previously,
this suggests \emph{arbitrary model manipulation} for $D \geq \tilde \Omega(\card{\setY} H^2 / P^2)$.

In both cases, 
two additional difficulty arises to turn these arguments into a rigorous proof.
First, gradient inversion is no longer guaranteed,
even though empirical studies successfully perform it some some contexts, 
by training neural nets,
often in the context of breaking privacy in federated learning~\cite{DBLP:conf/nips/JeonKLOO21,DBLP:conf/uai/WuCGW23}.
Second, the random vector subspaces that we discussed for neural nets 
are unlikely to be isotropically distributed.
Their distribution will depend on the particular model architecture and the training set.
Because of this, in practice, we expect general results to be very hard to obtain,
and the \emph{arbitrary model manipulation} regime to appear for larger values of $D$.

\subsection{Adding a regularization}
\label{sec:regularization}

It is common to add a regularization term, e.g. defining
$
    \calL (\alpha | \bD) 
    \triangleq \frac{\lambda}{2} \norm{\alpha}{2}^2
        + \frac{1}{N} \sum_{n \in [N]} \ell(\alpha | x_n, y_n)$.
A natural gradient descent robustification uses the following robust gradient estimate:
\begin{equation}
    \lambda \alpha 
        + \agg (\nabla \ell(\alpha | x_1, y_1), \ldots, \nabla \ell(\alpha | x_N, y_N)).
\end{equation}
Targeting $\alpha$ now requires manipulating $\agg{}$ into $-\lambda \alpha$.
This analysis is significantly harder than what we have,
as it breaks symmetries that are widely exploited by our proof.
Nevertheless, 
when $\norm{\alpha}{2} \ll 1/\lambda$, 
the effect of regularization vanishes, thereby suggesting manipulability towards $\alpha$.
Regularization can however rule out the manipulation towards values of $\alpha$
whose norm exceeds $1/\lambda$.

\section{Conclusion}
\label{sec:conclusion}

This paper provided security arguments for $D \leq 169 H^2 / P^2$.
Essentially, this is because, in high dimension, 
poisons can be made far more disinformative than honest data are informative.
In particular, numerical applications suggest that, to secure learning systems,
at least in many highly adversarial settings,
e.g. language processing and content recommendation,
model sizes should not exceed the billions,
even when the models are trained on ever more data.

Our results are far from conclusive though, and leave numerous important gaps to be filled.
For instance, it should be investigated whether 
\emph{arbitrary model manipulation} is specific to the defense with the geometric median and the clipped mean,
or if it could be generalized to all robust aggregation rules,
or to even more general defenses.
More importantly, despite our informal arguments,
the more general (and less tractable) case of nonlinear models, such as neural nets,
remains to be analyzed.

\bibliographystyle{alpha}
\bibliography{references}

\newpage
\appendix

\begin{center}
    \huge Appendix
\end{center}
Appendix~\ref{app:staterror} proves a simple closed form of the statistical error for linear regression.

Appendices~\ref{app:geomed} to~\ref{app:final} yield the proof of the main theorem (Theorem~\ref{th:negative}).
The proof is divided into four sections.
Appendix~\ref{app:geomed} analyzes the manipulability of the geometric median.
Appendix~\ref{app:expectation} proves that the expected gradient has a vanishingly small norm in high dimension.
Appendix~\ref{app:concentation} derives concentration bounds on the sum of normalized gradients.
Finally, Appendix~\ref{app:final} concludes.

Appendix~\ref{app:generalizations} provides a proof of Proposition~\ref{prop:subspace}.

\section{Statistical errors in linear regression}
\label{app:staterror}

The statistical error for linear regression has a simple closed form,
which we prove for a noisy true model $\calY_{\beta, \sigma} (x) = \calN(\beta^T x, \sigma^2)$.

\begin{proposition}
\label{prop:linear_regression_error}
    $\staterror (\alpha | \calY_{\beta, \sigma}) = \norm{\alpha - \beta}{2}^2 + \sigma^2$.
\end{proposition}

\begin{proof}
    Note that $\alpha^T x - y = (\alpha - \beta)^T x + (\beta^T x - y)$,
    which is the sum of two independent terms.
    Thus $\expect{ \norm{\alpha^T x - y}{2}^2 } = \expect{ ((\alpha - \beta)^T x)^2 } + \expect{ \varepsilon^2 }$,
    where $\varepsilon \triangleq \beta^T x - y$ is the random noise, whose distribution is $\calN(0, \sigma^2$.
    Clearly $\expect{ \varepsilon^2 } = \sigma^2$.
    
    The former term $\expect{ ((\alpha - \beta)^T x)^2 }$ is then $\norm{\alpha - \beta}{2}^2$ times the expectation of the square of the random coordinate of $x$ along the direction $\alpha - \beta$.
    By rotational invariance of $\calN(0, I_D)$, 
    we know that this random coordinate has the same distribution as any coordinate,
    which is $\calN(0, 1)$.
    Thus its expectation is the variance of $\calN(0, 1)$, which is one.
    Combining it all yields the lemma.
\end{proof}
\section{The manipulability of ``robust'' aggregation rules}
\label{app:geomed}

\subsection{Preliminaries about the geometric median}

Recall that the set of geometric medians is defined by
\begin{equation}
\geomed(z_1, \ldots, z_N) 
\triangleq \text{argmin}_{g \in \setR^d} \sum_{n \in [N]} \norm{g - z_n}{2}.
\end{equation}

To understand the vulnerability of \geomed{}-based gradient descent,
it is useful to study the first order condition.
For any vector $z \in \setR^d - \set{0}$, 
let us denote $\unit(z) \triangleq z/\norm{z}{2}$ the unit vector in the direction of $z$,
and $\unit_\Delta(z) \triangleq \min \set{1, \frac{\Delta}{\norm{z}{2}}} z$ the $\Delta$-clipping of $z$.
We also define $\unit(0) \triangleq \set{z \in \setR^d \st \norm{z}{2} \leq 1}$ as the unit ball.
We then have the following characterization.

\begin{proposition}
$G$ is a geometric median of $(z_1, \ldots, z_N)$ 
if and only if $0 \in \sum_{n \in [N]} \unit(G - z_n)$.
\end{proposition}

\begin{proof}
    This is the first order condition.
    By convexity of the Euclidean norm, this condition is sufficient.
\end{proof}

We now characterize the manipulability of the geometric median.

\begin{proposition}
\label{prop:geomed_hackability}
Consider $H$ vectors $z_1, \ldots, z_H \in \setR^d$.
Let $G \in \setR^d$ a target vector.
Then there exists $P$ vectors $z_{H+1}, \ldots, z_{H+P} \in \setR^d$
such that $G \in \geomed(z_1, \ldots, z_{H+P})$,
if and only if,
for each $h \in [H]$, there exists $u_h \in \unit(z_h - G)$ 
such that $\norm{ \sum_{h \in [H]} u_h }{2} \leq P$.
\end{proposition}

\begin{proof}
Assume $g \in \geomed(z_1, \ldots, z_{H+P})$.
Then $0 \in \sum_{n \in [H+P]} \unit(G - z_n)$.
Thus there exists $u_n \in \unit(G - z_n)$ for all $n \in [H+P]$
such that $\sum u_n = 0$.
But then, this implies that 
\begin{align}
\norm{\sum_{h \in [H]} u_h}{2} 
= \norm{- \sum_{p \in [P]} u_{H+p}}{2} 
\leq \sum_{p \in [P]} \norm{u_{H+p}}{2} \leq P.
\end{align}
This concludes the first implication.

Conversely, assume that $u_h \in \unit(z_h - G)$ for all $h \in [H]$
with $\norm{\sum u_h}{2} \leq P$.
Define $u^{(P)} \triangleq - \frac{1}{P} \sum u_h$.
We then have $\norm{u^{(P)}}{2} \leq 1$.
Now define $z_{H+p} \triangleq G$ for all $p \in [P]$.
Since $\unit(z_{H+p} - G) = \unit(0)$ which is the unit ball,
we know that $u_{H+p} \triangleq u^{(P)} \in \unit(z_{H+p} - G)$.
Thus $u_n \in \unit(z_n - G)$ for all $n \in [H+P]$,
with $\sum_{n \in [H+P]} u_n = \sum_{h \in H} u_h + P u^{(P)} = 0$.
This proves that $G \in \geomed(z_1, \ldots, z_{H+P})$,
and concludes the proof.
\end{proof}

\subsection{A necessary and sufficient condition for learning manipulability}

Combining Lemma~\ref{lemma:gradient_inversion} to Proposition~\ref{prop:geomed_hackability} 
now leads the simple characterization of vectors $\alpha$ that can be made stationary points
by a poisoning attack.
To state it, for $h \in [H]$, let us denote
$g_h(\alpha) \triangleq \nabla \ell(\alpha|x_h, y_h)$.

\begin{lemma}
\label{lemma:fundamental}
Let $\alpha \in \setR^D$.
Consider an honest dataset $\bH$, for logistic or linear regression.
There exists $\bP$ such that $\alpha$ is a $(\geomed, \bH \cup \bP)$-stationary point,
if and only if,
for all $h \in [H]$, there exists $u_h \in \unit(g_h(\alpha))$
such that $\norm{\sum_{h \in [H]} u_h}{2} \leq P$.
\end{lemma}

\begin{proof}
Recall that $\alpha$ is a $(\geomed, \bH \cup \bP)$-stationary point, if and only if, 
$0 \in \widehat{\nabla \calL} (\alpha | \bH \cup \bP)$,
whose right-hand side is
$\widehat{\nabla \calL} (\alpha | \bH \cup \bP) = \geomed\left( \nabla \ell(\alpha|x_1, y_1), \ldots, \nabla \ell(\alpha|x_{H}, y_{H}), \nabla \ell(\alpha|x_{H+1}, y_{H+1}), \ldots \nabla \ell(\alpha|x_{H+P}, y_{H+P}) \right)$.
Here, $G \triangleq 0$ acts like the target vector of Proposition~\ref{prop:geomed_hackability}.
This proposition implies that there exists $z_{H+1}, \ldots, z_{H+P}$ such that
$0 \in \geomed\left( \nabla \ell(\alpha|x_1, y_1), \ldots, \nabla \ell(\alpha|x_{H}, y_{H}), z_{H+1}, \ldots, z_{H+P} \right)$,
if and only if, 
$\sum \unit(g_h(\alpha))$ contains a vector of norm at most $P$.
Lemma~\ref{lemma:gradient_inversion} then asserts that 
the existence of $z_{H+1}, \ldots, z_{H+P}$ is equivalent to that of $\bP$.
\end{proof}

Note that, since $x_h$ is drawn from a continuous distribution, 
for any $\alpha$, the probability that $g_h(\alpha) = 0$ is zero.
Thus for a given $\alpha$, 
the condition under which it can be turned into a $(\bH \cup \bP, d)$-stationary point
can essentially be written $\norm{\sum_{h \in [H]} \unit(g_h(\alpha))}{2} \leq P$
(if $\bH$ is defined independently from $\alpha$).

We prove a similar result for $\Delta$-clipped mean.

\begin{lemma}
\label{lemma:clipmean_manipulability}
Let $\alpha \in \setR^D$.
Consider an honest dataset $\bH$, for logistic or linear regression.
There exists $\bP$ such that $\alpha$ is a $(\clipmean{\Delta}, \bH \cup \bP)$-stationary point,
if and only if $\norm{\sum_{h \in [H]} \unit_\Delta(g_h)}{2} \leq P \Delta$.
\end{lemma}

\begin{proof}
Assume $\alpha$ is a $(\clipmean{\Delta}, \bH \cup \bP)$-stationary point.
Then we must have
\begin{align}
    0 
    &= \norm{\sum_{h \in [H]} \unit_\Delta(g_h) + \sum_{p \in [P]} \unit_\Delta(g_{H+p})}{2} 
    \geq \norm{\sum_{h \in [H]} \unit_\Delta(g_h)}{2}
        - \sum_{p \in [P]} \norm{\unit_\Delta(g_{H+p})}{2} \\
    &\geq \norm{\sum_{h \in [H]} \unit_\Delta(g_h)}{2} - P \Delta.
\end{align}
Rearranging the terms yields the first implication.

Conversely, for $p \in [P]$, 
define $g_{H+p} \triangleq - \frac{1}{P} \sum_{h \in [H]} \unit_\Delta(g_h)$.
Then $\norm{\sum_{h \in [H]} \unit_\Delta(g_h)}{2} \leq P \Delta$ implies $\norm{g_{H+P}}{2} \leq \Delta$.
Thus 
\begin{align}
    &\sum_{h \in [H]} \unit_\Delta(g_h)
        + \sum_{p \in [P]} \unit_\Delta(g_{H+p}) 
    = \sum_{h \in [H]} \unit_\Delta(g_h)
        + \sum_{p \in [P]} g_{H+p} \\
    &= \sum_{h \in [H]} \unit_\Delta(g_h)
        - \sum_{p \in [P]} \frac{1}{P} \sum_{h \in [H]} \unit_\Delta(g_h) 
    = 0,
\end{align}
which implies that $\alpha$ is a $(\clipmean{\Delta}, \bH \cup \bP)$-stationary point.
\end{proof}

\section{The expectation of normalized/clipped gradient is small}
\label{app:expectation}

In this section, we prove that the expected value of an honest normalized/clipped gradient is small,
especially in very high dimension.
More precisely, denote $g \triangleq \nabla \ell(\alpha |x, y)$ a random gradient at $0$,
where $x \sim \calX$ is drawn from an isotropic distribution and $y \sim \calY_{\beta, \sigma}(x)$,
where $\calY$ corresponds to either linear or logistic regression.

We will prove that $\norm{\expect{\unit(g)}}{2} \leq 12 / \sqrt{D}$
and that $\norm{\expect{\unit_\Delta(g)}}{2} \leq 12 \Delta / \sqrt{D}$.

\subsection{Two infinite series bounds}

Before detailing the proof, let us first notice the two following bounds on two infinite series.

\begin{lemma}
\label{lemma:series1}
    $\sum_{k=1}^{\infty} (k+1) e^{-k^2 / 8} \leq 6$.
\end{lemma}

\begin{proof}
    First note that $\frac{d}{dt} (t e^{-t^2 / 8}) = (1 - \frac{1}{4} t^2) e^{-t^2/8}$.
    Thus, for $t \geq 2$, $t e^{-t^2 / 8}$ is decreasing.
    In particular, for $k \geq 3$ and $t \in [k-1, k]$, 
    we have $k e^{-k^2/8} \leq t e^{-t^2/8}$.
    Let $K \geq 3$. 
    Using also $k^2 \geq k$ for $k \geq K$, we have
    \begin{align}
        \sum_{k=1}^{\infty} (k+1) e^{-k^2 / 8}
        &= \sum_{k=1}^{K-1} (k+1) e^{-k^2 / 8}
        + \sum_{k=K}^{\infty} k e^{-k^2 / 8}
        + \sum_{k=K}^{\infty} e^{-k^2 / 8} \\
        &\leq \sum_{k=1}^{K-1} (k+1) e^{-k^2 / 8}
        + \sum_{k=K}^{\infty} \int_{k-1}^k t e^{-t^2 / 8} dt
        + \sum_{k=K}^{\infty} e^{-k / 8} \\
        &\leq \sum_{k=1}^{K-1} (k+1) e^{-k^2 / 8}
        - 4 \left[ e^{-t^2/8} \right]_{t=K-1}^\infty
        + \frac{e^{-K/8}}{1-e^{-1/8}} \\
        &= \sum_{k=1}^{K-1} (k+1) e^{-k^2 / 8}
        + 4 e^{-(K-1)^2/8}
        + \frac{e^{-K/8}}{1-e^{-1/8}}.
    \end{align}
    Plugging $K = 100$ in the computation above yields the lemma.
\end{proof}

\begin{lemma}
\label{lemma:series2}
    $\sum_{k=1}^{\infty} \frac{k+1}{k} e^{-k^2 / 2} \leq 1.44$.
\end{lemma}

\begin{proof}
    Note that $\frac{k+1}{k} \leq 2$ and $k^2 \geq k$ for $k \geq 1$.
    Let $K \geq 2$. Then
    \begin{align}
        \sum_{k=1}^{\infty} \frac{k+1}{k} e^{-k^2 / 2} 
        &\leq \sum_{k=1}^{K-1} \frac{k+1}{k} e^{-k^2 / 2} 
        + \sum_{k=K}^{\infty} \frac{k+1}{k} e^{-k^2 / 2} \\
        &\leq \sum_{k=1}^{K-1} \frac{k+1}{k} e^{-k^2 / 2} 
        + 2 \sum_{k=K}^{\infty} e^{-k / 2} \\
        &= \sum_{k=1}^{K-1} \frac{k+1}{k} e^{-k^2 / 2} 
        + \frac{2 e^{-K/2}}{1-e^{-1/2}}.
    \end{align}
    We conclude by taking $K = 13$.
\end{proof}

\subsection{A general form of the gradient}

\begin{lemma}
For both linear and logistic regression, $\unit(g) = S(x, \xi) \unit(x)$,
for some function $S$ with values in $\set{\set{+1}, \set{-1}, [-1, +1]}$ and a noise $\xi$ independent from $x$.
Moreover, $\unit_\Delta(g) = S_\Delta(x, \xi) \unit(x)$, 
where $S_\Delta$ takes values in $[- \Delta, \Delta]$.
\end{lemma}

\begin{proof}
In linear regression (with noise), we have $g = (\alpha^T x - y) x = ((\alpha - \beta)^T x + \xi) x$, 
with some noise $\xi \sim \calN(0, \sigma^2)$ that is independent from $x$. 
This implies $\unit(g) = S_{linear}(x, \xi) \unit(x)$,
with $S_{linear}(x, \xi) \triangleq \sign\left((\alpha - \beta)^T x + \xi \right)$.
Moreover we can write $\unit_\Delta(g) = S_{\Delta, linear} (x, \xi) \unit(x)$,
where $S_{\Delta, linear} (x, \xi) = S_{linear}(x, \xi) \min \set{ 1, \frac{\Delta}{\absv{(\alpha - \beta)^T x + \xi} \cdot \norm{x}{2}} }$ if $(\alpha - \beta)^T x + \xi \neq 0$,
and $S_{\Delta, linear} (x, \xi) = 0$ otherwise.

For logistic regression, we have $g = (\sigma(\alpha^T x) - y) x$,
where the label $y$ is given by $y = \bbOne{\xi < \sigma(\beta^T x)}$,
for a random variable $\xi \sim \calU([0, 1])$ independent from $x$.
It follows that $\unit(g) = S_{logistic}(x, \xi) \unit(x)$, 
with $S_{logistic}(x, \xi) \triangleq \sign\left( \sigma(\alpha^T x) - \bbOne{\xi < \sigma(\beta^T x)} \right)$.
Again, we can write $\unit_\Delta(g) = S_{\Delta, logistic} (x, \xi) \unit(x)$,
where $S_{\Delta, logistic} (x, \xi) = S_{logistic}(x, \xi) \min \set{ 1, \frac{\Delta}{\absv{\sigma(\alpha^T x) - \bbOne{\xi < \sigma(\beta^T x)}} \cdot \norm{x}{2}} }$.
\end{proof}

\subsection{Two classical concentration bounds}

Before the sequel, we recall
the two following concentration bounds on the random normal vectors and on its norm.

\begin{lemma}
\label{lemma:normal_tail_bound}
Let $Z \sim \calN(0, 1)$ and $\kappa > 0$. 
Then $\pb{Z \geq \kappa} \leq \frac{1}{\kappa \sqrt{\tau}} e^{-\kappa^2 / 2}$.
\end{lemma}

\begin{proof}
Considering the variable change $s = t + \kappa$, and denoting $\tau = 6.28318530718...$, yields
\begin{align}
    \pb{Z \geq \kappa}
    &= \frac{1}{\sqrt{\tau}} \int_\kappa^\infty e^{-t^2/2} dt 
    = \frac{1}{\sqrt{\tau}} \int_0^\infty e^{-\kappa^2/2} e^{-\kappa s} e^{-s^2 /2} ds 
    \leq \frac{e^{-\kappa^2/2}}{\sqrt{\tau}} \int_0^\infty e^{-\kappa s} ds
    = \frac{e^{-\kappa^2/2}}{\sqrt{\tau}} \left[ \frac{e^{-\kappa s}}{\kappa} \right]_{s=0}^\infty,
\end{align}
where we used $e^{-s^2 /2}$.
Finishing the computation yields the lemma.
\end{proof}

\begin{lemma}[\cite{wainwright2019high}, Example 2.11]
\label{lemma:chi2}
Let $x \sim \calN(0, I_D)$ and $\kappa \geq 0$. 
Then
\begin{equation}
\pb{\norm{x}{2}^2 - D \geq \kappa \sqrt{D}} \leq 
\left\lbrace
\begin{array}{cc}
    e^{-\kappa^2/8}, & \text{for } \kappa \leq \sqrt{D},  \\
    e^{-\kappa \sqrt{D} / 8}, & \text{for } \kappa \geq \sqrt{D},
\end{array}
\right.
\end{equation}
and similarly for the event $\set{\norm{x}{2}^2 - D \leq - \kappa \sqrt{D}}$
\end{lemma}

\begin{proof}[Sketch of proof]
    The square of a standard normal variable is sub-exponential of parameters $(2, 4)$ (see Lemma~\ref{lemma:chi2_subexponential}),
    Thus the sum of $D$ of them is sub-exponential of parameters $(2\sqrt{D}, 4)$.
    The result follows from the sub-exponential tail bound.
\end{proof}

\subsection{An isotropic unit vector is almost orthogonal to any fixed vector}

In this section , we prove that the norm of $\mu$ is small.
Essentially, the intuition behind this is that $\unit(g)$ is unlikely to point 
towards its expected value $\mu$,
because high-dimensional isotropic vectors are nearly orthogonal.

To precisely derive the result, let us first identify the following high-probability events.

\begin{lemma}
\label{lemma:norm}
Let $\kappa \in [0, \sqrt{D}]$.
With probability at least $1-e^{-\kappa^2/8}$,
we have $\norm{x}{2}^2 \geq D - \kappa \sqrt{D}$.
\end{lemma}

\begin{proof}
This follows from Lemma~\ref{lemma:chi2} and $x \sim \calN(0,I_D)$.
\end{proof}

\begin{lemma}
\label{lemma:orthogonal}
Let $\kappa \geq 1$ and $u$ any unit vector. 
Then with probability at least $1- \frac{2}{\kappa \sqrt{\tau}} e^{-\kappa^2/2}$,
we have $\absv{u^T x} \leq \kappa$.
\end{lemma}

\begin{proof}
$u^T x \sim \calN(0, u^T I u) = \calN(0, 1)$.
Lemma~\ref{lemma:normal_tail_bound} yields the result.
\end{proof}

The two previous well-known facts allow to derive the following key lemma.

\begin{lemma}
\label{lemma:fundamental_isotropy}
    Consider a fixed unit vector $u \in \setS{D-1}$ 
    and assume $v$ is a random vector uniformly distributed over $\setS{D-1}$.
    Let $S$ be a random variable, which may depend on $v$, and with values in $[- \Delta, \Delta]$.
    For $D \geq 1024$, we have $\expect{S u^T v} \leq 12 \Delta / \sqrt{D}$.
\end{lemma}

\begin{proof}
    Note that $v$ can be written $x / \norm{x}{2}$ for $x \sim \calN(0, I_D)$.
For $\kappa \in [1, \sqrt{D}]$, define 
\begin{equation}
\event_\kappa \triangleq 
\set{ \norm{x}{2}^2 \geq D - \kappa \sqrt{D} }
\cap \set{ \absv{x^T u} \leq \kappa }.
\end{equation}
Combining the two previous lemma yields 
$\pb{\neg \event_\kappa} \leq e^{-\kappa^2 / 8} + \frac{2}{\kappa \sqrt{\tau}} e^{-\kappa^2 / 2}$.
Now define $K \triangleq \lfloor \sqrt{D}/4 \rfloor$.
For $\kappa \in [1, K]$, we then have $D - \kappa \sqrt{D} \geq 3D/4$.
Note moreover that 
$K^2 \geq (\frac{\sqrt{D}}{4} - 1)^2 \geq \frac{D}{16} - \frac{\sqrt{D}}{2}
\geq \frac{D}{16} \left(1 - \frac{8}{\sqrt{D}}\right)
\geq \frac{D}{32}$, using $D \geq 256$.
For $D \geq 1024$, we also have $e^{-D/256} \leq 0.587/\sqrt{D}$.
Moreover, note that $D - \sqrt{D} \geq (1- \frac{1}{\sqrt{D}}) D \geq (1-\frac{1}{\sqrt{1024}})D = \frac{31}{32} D$.
We now have the following union bound:
\begin{align}
&\expect{\frac{\absv{S u^T v}}{\Delta}}
= \sum_{\kappa = 1}^{K} 
  \expect{ \frac{\absv{ S x^T u }}{\Delta \norm{x}{2}}\st \event_\kappa - \event_{\kappa -1} }
  \pb{\event_\kappa - \event_{\kappa -1}} 
  + \expect{ \frac{\absv{ S x^T u }}{\Delta \norm{x}{2}}\st - \event_{K} }
  \pb{- \event_{K}} \\
&\leq \underbrace{\frac{1}{\sqrt{D - \sqrt{D}}}}_{\kappa = 1} + 
  \sum_{\kappa = 2}^{K} 
  \left( \frac{\kappa}{\sqrt{D - \kappa \sqrt{D}}} \right)
  \pb{- \event_{\kappa -1}} 
  + \underbrace{\expect{ \absv{\unit(x)^T u} \st - \event_{K} }}_{\leq 1}
  \pb{- \event_{K}} \\
&\leq 
  \sqrt{\frac{32}{31 D}} + 
  \sqrt{\frac{8}{7 D}} \sum_{\kappa = 2}^{K} 
  \kappa \left( e^{- \frac{(\kappa-1)^2}{8}} + \frac{2}{(\kappa - 1) \sqrt{\tau}} e^{-\frac{(\kappa-1)^2}{2}} \right)
  + e^{-\frac{K^2}{8}} + \frac{2}{K\sqrt{\tau}} e^{-\frac{K^2}{2}} \\
&\leq 
  \sqrt{\frac{32}{31 D}} + 
  \frac{1}{\sqrt{D}} \sqrt{\frac{8}{7}} \sum_{k=1}^{\infty} (k+1) e^{- \frac{k^2}{8}} 
  + \frac{1}{\sqrt{D}} \sqrt{\frac{8}{7 \tau}} \sum_{k=1}^{\infty} \frac{k+1}{k} e^{- \frac{k^2}{2}} 
  + e^{- \frac{D}{256}} + \frac{2\sqrt{32}}{\sqrt{\tau D}} e^{- \frac{D}{64}} \\
&\leq \frac{1}{\sqrt{D}} \left(
    \sqrt{\frac{32}{31}} 
    + \sqrt{\frac{8}{7}} \sum_{k=1}^{\infty} (k+1) e^{- \frac{k^2}{8}} 
    + \sqrt{\frac{8}{7 \tau}} \sum_{k=1}^{\infty} \frac{k+1}{k} e^{- \frac{k^2}{2}} 
    + 0.587
    + \frac{8\sqrt{2}}{\sqrt{\tau}} e^{-16}
  \right),
\end{align}
We now use $\sum_{k=1}^{\infty} (k+1) e^{-k^2 / 8} \leq 6$ (Lemma~\ref{lemma:series1}) 
and $\sum_{k=1}^{\infty} \frac{k+1}{k} e^{-k^2 / 2} \leq 1.44$ (Lemma~\ref{lemma:series2}),
which eventually yields the bound.
This then implies $\expect{S u^T v} \leq \expect{\absv{S u^T v}} \leq 12 \Delta / \sqrt{D}$.
\end{proof}

\subsection{The norms of the expected normalized/clipped gradient are small}

We now derive the main results of this section,
which bound the expected norms of normalized and clipped gradients.

\begin{lemma}
\label{lemma:bound_rho_d}
Assume $D \geq 1024$. 
Then $\norm{\expect{\unit(g)}}{2} \leq 12 / \sqrt{D}$.
\end{lemma}

\begin{proof}
    We apply Lemma~\ref{lemma:fundamental_isotropy} with $u \triangleq \unit(\expect{\unit(g)})$,
    $v = \unit(x)$, $S = S(x, \xi)$ and $\Delta = 1$.
    Recall that $\unit(g) = Sv$.
    We then have
    \begin{align}
        \norm{\expect{\unit(g)}}{2}
        &= \expect{\unit(g)}^T \unit\left(\expect{\unit(g)}\right)
        = \expect{S v}^T u
        = \expect{S u^T v}
        \leq 12 / \sqrt{D}
    \end{align}
    This concludes the proof.
\end{proof}

\begin{lemma}
\label{lemma:bound_rho_d_Delta}
Assume $D \geq 1024$. 
Then $\norm{\expect{\unit_\Delta (g)}}{2} \leq 12 \Delta / \sqrt{D}$.
\end{lemma}

\begin{proof}
    We apply Lemma~\ref{lemma:fundamental_isotropy} with $u \triangleq \unit(\expect{\unit_\Delta(g)})$,
    $v = \unit(x)$ and $S = S_\Delta(x, \xi)$.
    Recall that $\unit_\Delta(g) = Sv$.
    We then have
    \begin{align}
        \norm{\expect{\unit_\Delta(g)}}{2}
        &= \expect{\unit_\Delta(g)}^T \unit\left(\expect{\unit_\Delta(g)}\right)
        = \expect{S v}^T u
        = \expect{S u^T v}
        \leq 12 \Delta / \sqrt{D}
    \end{align}
    This concludes the proof.
\end{proof}

\section{Isotropic unit/clipped vectors are sub-Gaussian}
\label{app:concentation}

In this section we show that isotropic unit vectors are sub-Gaussian,
which will allow us to derive a concentration bound on the deviation to the expectation.

\subsection{Preliminaries on sub-Gaussian variables}

\begin{definition}
A random zero-mean variable $X$ is sub-Gaussian of parameter $\sigma$ if
\begin{equation}
\forall \lambda \in \setR \mathsep 
\expect{e^{\lambda X}} \leq e^{\lambda^2 \sigma^2 /2}.
\end{equation}
\end{definition}

\begin{lemma}[Adapted from~\cite{vershynin2018high}, Proposition 2.5.2]
\label{lemma:tail_bound_implies_sub_gaussian}
    Suppose $X$ is a random zero-mean variable that verifies $\pb{\absv{X} \geq t} \leq 3 e^{-t^2/c^2}$ for all $t \geq 0$.
    Then $X$ is sub-Gaussian of parameter $14 c$.
\end{lemma}

\begin{proof}
    For any $p \geq 1$, we have
    \begin{align}
        \expect{\absv{X}^p}
        &= \int_0^\infty \pb{\absv{X}^p \geq s} ds 
        \overset{t=s^p}{=} \int_0^\infty \pb{\absv{X} \geq t} p t^{p-1} dt \\
        &\leq 3p \int_0^\infty e^{-t^2/c^2} t^{p-1} dt
        \overset{u=t^2/c^2}{=} 3p \int_0^\infty e^{-u} c^{p-1} u^{\frac{p-1}{2}} \frac{c^2 du}{c \sqrt{u}} \\
        &\leq 3 p c^p \int_0^\infty e^{-u} u^{\frac{p}{2}-1} du
        = 3p c^p \Gamma(p/2),
    \end{align}
    where $\Gamma$ is the classical gamma function.
    It is well-known that for $t \geq 1/2$, it verifies $\Gamma(t) \leq 3t^t$.
    Thus $\expect{\absv{X}^p} \leq 9 p c^p (p/2)^{p/2}$.

    Now let $\nu \geq 0$.
    The Taylor expansion of the exponential function yields
    \begin{align}
        \expect{e^{\nu^2 X^2}}
        &= \expect{ 1 + \sum_{q=1}^\infty \frac{(\nu^2 X^2)^q}{q!} }
        = 1 + \sum_{q=1}^\infty \frac{\nu^{2q}}{q!} \expect{\absv{X}^{2q}} 
        \leq 1 + \sum_{q=1}^\infty \frac{\nu^{2q}}{(q/e)^q} (18 q c^{2q} q^q),
    \end{align}
    using Stirling approximation $q! \geq (q/e)^q$.
    We now use the bound $18q \leq e^{2.5 q}$ for $q \geq 1$, which implies
    \begin{align}
        \expect{e^{\nu^2 X^2}}
        &\leq 1 + \sum_{q=1}^\infty \nu^{2q} (e^{3.5 q} c^{2q})
        \leq \sum_{q=0}^\infty \left( 36 \nu^2 c^2 \right)^q
        = \frac{1}{1 - 6^2 \nu^2 c^2},
    \end{align}
    for $\nu < 1/6 c$.
    We now use the bound $\frac{1}{1-t} \leq e^{2t}$ for $t \in [0, 1/2]$.
    Thus, 
    
    \begin{equation}
    \label{eq:subgaussian_lambda_bounded}
        \forall \nu \in \left[ 0, \frac{1}{6 c \sqrt{2}} \right] \mathsep
        \expect{e^{\nu^2 X^2}} 
        \leq \exp(72 \nu^2 c^2).
    \end{equation}

    Now consider $\lambda \in \setR$.
    We now use the inequality $e^t \leq t + e^{t^2}$ for all $t \in \setR$.
    Then, for all $\lambda \in \setR$,
    \begin{align}
        \expect{e^{\lambda X}}
        \leq \expect{\lambda X + e^{\lambda^2 X^2}}
        = \lambda \underbrace{\expect{X}}_{=0} + \expect{e^{\lambda^2 X^2}}
        = \expect{e^{\lambda^2 X^2}}.
    \end{align}
    Now, for $\absv{\lambda} \leq \frac{1}{6 c \sqrt{2}}$, 
    the previous bound implies $\expect{e^{\lambda X}} \leq e^{72 \lambda^2 c^2}
    \leq e^{90 \lambda^2 c^2}$.
    
    For $\absv{\lambda} \geq \frac{1}{6 c \sqrt{2}}$,
    we use the inequality $2 \lambda t \leq \alpha \lambda^2 + \frac{1}{\alpha} t^2$ (for all $\lambda, t$ and $\alpha > 0$)
    with $\alpha = 36 c^2$,
    yielding $\expect{e^{\lambda X}} \leq e^{18 c^2 \lambda^2} \expect{e^{X^2/72c^2}}$.
    Now note that we can apply Equation~\eqref{eq:subgaussian_lambda_bounded}
    for $\nu = 1/6c\sqrt{2}$ (and thus $\nu^2 = 1/72 c^2$),
    which gives $\expect{e^{X^2/72c^2}} \leq e$.
    For $\absv{\lambda} \geq \frac{1}{6 c \sqrt{2}}$, 
    we have $e \leq e^{72 c^2 \lambda^2}$,
    which then implies $\expect{e^{\lambda X}} \leq e^{(18 + 72) c^2 \lambda^2} = e^{90 c^2 \lambda^2}$.

    Overall we thus have the sub-Gaussian bound, for $\sigma^2 = 180 c^2$.
    The bound $\sqrt{180} \leq 14$ allows to conclude.
\end{proof}

\begin{lemma}
\label{lemma:addition_subgaussian}
    If $X$ and $Y$ are independant sub-Gaussians of parameters $\sigma_X$ and $\sigma_Y$,
    then $X+Y$ is sub-Gaussian of parameter $\sqrt{\sigma_X^2 + \sigma_Y^2}$.
\end{lemma}

\begin{proof}
    Let $\lambda \in \setR$.
    Then $\expect{e^{\lambda (X+Y)}} = \expect{e^{\lambda X} e^{\lambda Y}} 
    = \expect{e^{\lambda X}} \expect{e^{\lambda Y}} 
    \leq e^{\lambda^2 \sigma_X^2 /2} e^{\lambda^2 \sigma_Y^2 /2}
    = e^{\lambda^2 (\sigma_X^2 + \sigma_Y^2) /2}$
\end{proof}

\begin{lemma}[\cite{wainwright2019high}, Exercise 2.4]
\label{lemma:bounded_subgaussian}
    If $X$ has zero mean and verifies $a \leq X \leq b$.
    Then $X$ is sub-Gaussian with parameter $\frac{b-a}{2}$.
\end{lemma}

\begin{proof}
    We refer to \cite{wainwright2019high}, Exercise 2.4,
    as obtaining the tight parameter is technical and not very interesting for our purpose.
\end{proof}

\begin{theorem}[Hoeffding~\cite{wainwright2019high}]
\label{th:hoeffding}
Suppose $X$ is a zero-mean sub-Gaussian variable of parameter $\sigma$.
Then
\begin{equation}
\forall t \geq 0 \mathsep 
\pb{X \geq t} \leq e^{-t^2 / 2 \sigma^2}.
\end{equation}
\end{theorem}

\begin{proof}
    This follows from the classical Chernoff bound. 
    Namely, for $\lambda \geq 0$, we have
    \begin{align}
        \pb{X \geq t}
        &\leq e^{-\lambda t} \expect{e^{\lambda X}}
        \leq e^{-\lambda t} e^{\lambda^2 \sigma^2 / 2}.
    \end{align}
    We conclude with the special case $\lambda = t / \sigma^2$.
\end{proof}

\subsection{Bounding the tail distribution}

\begin{lemma}
\label{lemma:unit_gradient_tail_bound}
    For any unit vector $u$ and $\kappa \geq 0$,
    $\pb{ \absv{\unit(g)^T u} \geq \kappa} \leq 3 e^{-\kappa^2 D / 4}$.
\end{lemma}

\begin{proof}
    Note that the lemma is trivially true for $\kappa \leq \sqrt{2/D}$,
    as the right-hand side is then at least $3e^{-1/2} \geq 1$.
    Without loss of generality, we now assume $\kappa \geq \sqrt{2/D}$.

    We can write $\unit(g) =\chi \frac{ x }{\norm{x}{2}}$, 
    with $x \sim \calN(0, I_D)$ and $\chi \in \set{-1, +1}$,
    whose distribution depends on $x$.

    Now denote $\event \triangleq \set{ \norm{x}{2} \geq \sqrt{\frac{D}{2}} }$.
    Note that $\event$ holds if and only if $\norm{x}{2}^2 \geq \frac{1}{2} D$.
    For $\event$ to fail, we must have $\norm{x}{2}^2 < \frac{1}{2} D$,
    which requires $\norm{x}{2}^2 - D < - \frac{1}{2} D$.
    By Lemma~\ref{lemma:chi2} (with $\kappa = \sqrt{D}/2$), we know that
    $\pb{\event} \geq 1 - e^{-D/32}$.

    Consider any unit vector $u$.
    \begin{align}
        \pb{\absv{\unit(g)^T u} \geq \kappa}
        &\leq \pb{\absv{x^T u} \geq \kappa \norm{x}{2} ~\text{and}~ \calE} + \pb{\neg \calE} \\
        &\leq \pb{\absv{x^T u} \geq \kappa \sqrt{\frac{D}{2}}} + e^{-D/32} 
        \leq 2 e^{-\kappa^2 D/4} + e^{-D/32},
    \end{align}
    using $x^T u \sim \calN(0, 1)$ and Lemma~\ref{lemma:normal_tail_bound} (with $\kappa \geq \sqrt{2/D}$).
    Now, if $\kappa \leq 1$, 
    then $e^{-\kappa^2 D/4} \geq e^{-D/4} \geq e^{-D/32}$,
    and we thus have the bound $\pb{\absv{\unit(g)^T u} \geq \kappa} \leq 3 e^{-\kappa^2 D/4}$.    
    For $\kappa \geq 1$, we trivially have the same bound,
    since $\absv{\unit(g)^T u}$ is necessarily at most 1 (scalar product of unit vectors).
\end{proof}

\begin{lemma}
\label{lemma:unit_delta_gradient_tail_bound}
    For any unit vector $u$ and $\kappa \geq 0$,
    $\pb{ \absv{\unit_\Delta(g)^T u} \geq \kappa \Delta} \leq 3 e^{-\kappa^2 D / 4}$.
\end{lemma}

\begin{proof}
    Note that $\absv{\unit_\Delta(g)^T u} \leq \Delta \absv{\unit(g)^T u}$.
    It then suffices to apply Lemma~\ref{lemma:unit_gradient_tail_bound}.
\end{proof}

\begin{lemma}
\label{lemma:mu_perp_subgaussian}
    Let $u$ and $v$ be unit vectors respectively orthogonal to $\expect{\unit(g)}$ and $\expect{\unit_\Delta(g)}$.
    Then $\unit(g)^T u$ and $\frac{1}{\Delta} \unit_\Delta(g)^T v$ are both zero-mean sub-Gaussian variables with parameter $28/\sqrt{D}$.
\end{lemma}

\begin{proof}
    $\expect{\unit(g)^T u} = \expect{\unit(g)}^T u = 0$,
    thus $\unit(g)^T u$ has zero mean.
    Similarly $\expect{\unit_\Delta(g)^T v} = \expect{\unit_\Delta(g)}^T v = 0$,
    thus $\frac{1}{\Delta} \unit_\Delta(g)^T v = 0$.
    By Lemma~\ref{lemma:unit_gradient_tail_bound} 
    and Lemma~\ref{lemma:tail_bound_implies_sub_gaussian} (with $c = 2/\sqrt{D}$),
    we know that $\unit(g)^T u$ is sub-Gaussian with parameter $14 \cdot 2/\sqrt{D}$,
    and so is $\frac{1}{\Delta} \unit_\Delta(g)^T v$ (using Lemma~\ref{lemma:unit_delta_gradient_tail_bound}).
\end{proof}

\begin{lemma}
\label{lemma:mu_subgaussian}
    $(\unit(g) - \expect{\unit(g)})^T \unit(\expect{\unit(g)})$ 
    and $\frac{1}{\Delta} (\unit_\Delta(g) - \expect{\unit_\Delta(g)})^T \unit(\expect{\unit_\Delta(g)})$ 
    are zero-mean sub-Gaussian variables with parameter $1$.
\end{lemma}

\begin{proof}
    Denote $\mu \triangleq \expect{\unit(g)}$.
    $\expect{(\unit(g) - \mu)^T \unit(\mu)} = (\expect{\unit(g)} - \mu)^T \unit(\mu) = 0$,
    thus $(\unit(g) - \mu)^T \unit(\mu)$ has mean zero.
    Moreover $(\unit(g) - \mu)^T \unit(\mu) = \unit(g)^T \unit(\mu) - \norm{\mu}{2} \in [-1 - \norm{\mu}{2}, 1 - \norm{\mu}{2}]$.
    By Lemma~\ref{lemma:bounded_subgaussian}, 
    it is thus sub-Gaussian of parameter $1$.
    A similar argument applies to $\frac{1}{\Delta} (\unit_\Delta(g) - \expect{\unit_\Delta(g)})^T \unit(\expect{\unit_\Delta(g)})$.
\end{proof}

\subsection{A concentration bound on the sum of normalized honest gradients}

\begin{lemma}    
\label{lemma:event_negative}
Let $\kappa \geq 0$. 
With probability at least $1 - 2 D e^{-\kappa^2 / 1568}$,
\begin{equation}
\norm{\sum_{h \in [H]} \unit(g_{h}) - H \expect{\unit(g)}}{2} \leq \kappa \sqrt{2 H}
\end{equation}
\end{lemma}

\begin{proof}
Consider an orthonormal basis $\basis_1, \ldots, \basis_{D-1}$ of the hyperplane $\mu^\perp$ orthogonal to $\mu \triangleq \expect{\unit(g)}$.
From Lemma~\ref{lemma:mu_perp_subgaussian}, along each coordinate $i \in [D-1]$, 
$\unit(g_{h})^T \basis_i$ is sub-Gaussian of parameter $28/\sqrt{D}$.
Since these terms are independent, 
by Lemma~\ref{lemma:addition_subgaussian},
we know that $X_i \triangleq \sum_h \unit(g_{h})^T \basis_i$ is sub-Gaussian of parameter $28 \sqrt{H/D}$.
By Theorem~\ref{th:hoeffding} (with $t = \kappa \sqrt{H/D}$),
we know that $\absv{X_i} \leq \kappa \sqrt{H/D}$ with probability at least $1 - 2e^{-\kappa^2/2\cdot 28^2}$.

Now consider $X_\mu \triangleq \sum_h (\unit(g_h) - \mu)^T \unit(\mu)$.
This is the sum of $H$ independent variables, each of which is sub-Gaussian of parameter $1$.
Thus $X_\mu$ is sub-Gaussian of parameter $\sqrt{H}$.
By Theorem~\ref{th:hoeffding} (with $t = \kappa \sqrt{H}$),
we have $\absv{X_\mu} \leq \kappa \sqrt{H}$ with probability at least $1 - 2e^{-\kappa^2/2} \geq 1 - 2^{-\kappa^2 / 2 \cdot 28^2}$.

Taking the intersection of all these events imply that,
with probability at least $1 - 2 D e^{-\kappa^2 / 1568}$,
we have
\begin{equation}
    \norm{\sum_{h \in [H]} \unit(g_{h}) - H \mu}{2}^2
    = X_\mu^2 + \sum_{i=1}^{D-1} X_i^2
    \leq \kappa^2 H + (D-1) \kappa^2 \frac{H}{D} \leq 2 \kappa^2 H.
\end{equation}
Taking the square root concludes the proof.
\end{proof}

\begin{lemma}    
\label{lemma:event_negative_delta}
Let $\kappa \geq 0$. 
With probability at least $1 - 2 D e^{-\kappa^2 / 1568}$,
\begin{equation}
\norm{\sum_{h \in [H]} \unit_\Delta(g_{h}) - H \expect{\unit_\Delta(g)}}{2} \leq \kappa \Delta \sqrt{2 H}
\end{equation}
\end{lemma}

\begin{proof}
Consider an orthonormal basis $\basis_1, \ldots, \basis_{D-1}$ of the hyperplane $\mu_\Delta^\perp$ orthogonal to $\mu_\Delta \triangleq \expect{\unit_\Delta(g)}$.
From Lemma~\ref{lemma:mu_perp_subgaussian}, along each coordinate $i \in [D-1]$, 
$\frac{1}{\Delta} \unit_\Delta(g_{h})^T \basis_i$ is sub-Gaussian of parameter $28/\sqrt{D}$.
Since these terms are independent, 
by Lemma~\ref{lemma:addition_subgaussian},
we know that $X_i \triangleq \sum_h \frac{1}{\Delta} \unit_\Delta(g_{h})^T \basis_i$ is sub-Gaussian of parameter $28 \sqrt{H/D}$.
By Theorem~\ref{th:hoeffding} (with $t = \kappa \sqrt{H/D}$),
we know that $\absv{X_i} \leq \kappa \sqrt{H/D}$ with probability at least $1 - 2e^{-\kappa^2/2\cdot 28^2}$.

Now consider $X_\mu \triangleq \frac{1}{\Delta} \sum_h (\unit_\Delta(g_h) - \mu_\Delta)^T \unit(\mu_\Delta)$.
This is the sum of $H$ independent variables, each of which is sub-Gaussian of parameter $1$.
Thus $X_\mu$ is sub-Gaussian of parameter $\sqrt{H}$.
By Theorem~\ref{th:hoeffding} (with $t = \kappa \sqrt{H}$),
we have $\absv{X_\mu} \leq \kappa \sqrt{H}$ with probability at least $1 - 2e^{-\kappa^2/2} \geq 1 - 2^{-\kappa^2 / 2 \cdot 28^2}$.

Taking the intersection of all these events imply that,
with probability at least $1 - 2 D e^{-\kappa^2 / 1568}$,
we have
\begin{equation}
    \norm{\sum_{h \in [H]} \frac{1}{\Delta} \unit_\Delta(g_{h}) - H \frac{\mu_\Delta}{\Delta}}{2}^2
    = X_\mu^2 + \sum_{i=1}^{D-1} X_i^2
    \leq \kappa^2 H + (D-1) \kappa^2 \frac{H}{D} \leq 2 \kappa^2 H.
\end{equation}
Taking the square root, and multiplying by $\Delta$, concludes the proof.
\end{proof}

\section{Proof of arbitrary model manipulation (Theorem~\ref{th:negative})}
\label{app:final}

In this section, we prove Theorem~\ref{th:negative}.
By Lemma~\ref{lemma:fundamental},
it suffices to prove that the sum $\sum \unit(g_h)$ 
of normalized honest gradients has a norm bounded by $P$,
and that $\sum \unit_\Delta (g_h)$ has a norm bounded by $\Delta P$.

\begin{proof}
Consider the event of Lemma~\ref{lemma:event_negative} with $\kappa = 56 \sqrt{\ln D}$.
Then with probability at least $1 - 2 D e^{-2\ln D} = 1 - 2/D \geq 1 - \delta$, we have
\begin{align}
\norm{\sum_{h \in [H]} \unit(g_{h})}{2}
&\leq \norm{\sum_{h \in [H]} \unit(g_{h}) - H \expect{\unit(g)}}{2}
  + \norm{H \expect{\unit(g)}}{2} \\
&\leq 56 \sqrt{2 H \ln D} + \frac{12 H }{\sqrt{D}},
\end{align}
using Lemma~\ref{lemma:bound_rho_d}.
Assuming $H \geq 6272 D \ln D = 2 \cdot 56^2 D \ln D$ implies $56 \sqrt{H \ln D} \leq H/\sqrt{D}$.
Thus$\norm{\sum_{h \in [H]} \unit(g_{h})}{2} \leq 13 H / \sqrt{D} \leq P$.
By Lemma~\ref{lemma:fundamental}, we know that $\alpha$ can be made a stationary point.
The case of $\unit_\Delta$ is dealt with similarly.
\end{proof}

\section{On isotropically random subspaces (Proposition~\ref{prop:subspace})}
\label{app:generalizations}

The proof of Proposition~\ref{prop:subspace} leverages well-known concentration bounds on sub-exponential distributions.

\subsection{Sub-exponential distributions and bounds}

\begin{definition}
A zero-mean random variable $X$ with mean $\mu$ is sub-exponential with parameters $(\nu, b)$ if
\begin{equation}
\forall \absv{\lambda} < 1/b \mathsep
\expect{e^{\lambda(X-\mu)}} \leq e^{\nu^2 \lambda^2 / 2}.
\end{equation}
\end{definition}

\begin{lemma}[\cite{wainwright2019high}, Example 2.8]
\label{lemma:chi2_subexponential}
Let $Z \sim \calN(0,1)$. 
Then $Z^2$ is sub-exponential with parameters $(\nu = 2, b = 4)$.
\end{lemma}

\begin{proof}
This follows from explicitly computing $\expect{ e^{\lambda (Z^2 - 1)} }$, 
and by bounding the result appropriately.
\end{proof}

\begin{lemma}
\label{lemma:product_normal_subexponential}
Suppose $X, Y \sim \calN(0, 1)$ are independent.
Then $XY$ is zero mean and sub-exponential of parameters $(2\sqrt{2}, 4)$.
\end{lemma}

\begin{proof}
We have $XY = A^2 - B^2$,
where $A \triangleq \frac{X+Y}{\sqrt{2}}$ and $B \triangleq \frac{X-Y}{\sqrt{2}}$.
Since $(A, B)$ results from rotating $(X, Y)$ by an eighth of a turn,
given the rotational symmetry of $\calN(0, I_2)$,
it is known that $(A, B)$ also follows the $\calN(0, I_2)$ distribution.
In particular, $A^2$ and $B^2$ are each the square of a standard centered normal distribution,
and they are independent.
By Lemma~\ref{lemma:chi2_subexponential}, it follows that $A^2$ and $B^2$ each has zero mean,
and is sub-exponential with parameters $(2, 4)$.
Moreover, by Lemma~\ref{lemma:addition_of_subexponentials}, 
their difference is also zero mean and is sub-exponential, of parameters $(2\sqrt{2}, 4)$.
\end{proof}

\begin{lemma}
\label{lemma:addition_of_subexponentials}
Assume that for each $i \in [N]$, 
$X_i$ is an independent zero-mean sub-exponential random variable with parameters $(\nu_i, b_i)$.
Then $\sum_{i \in [N]} X_i$ is zero-mean sub-exponential with parameters $(\sqrt{\sum \nu_i^2}, \max b_i)$.
\end{lemma}

\begin{proof}
Without loss of generality, assume that each $X_i$ has zero mean.
Let $\lambda < 1/(\max b_i) = \min (1/b_i)$.
Then
\begin{equation}
\expect{ e^{\lambda \sum X_i} }
= \prod_{i \in [N]} \expect{ e^{\lambda X_i} }
\leq \prod_{i \in [N]} e^{\nu_i^2 \lambda^2 / 2}
= e^{\left(\sum \nu_i^2\right) \lambda^2 / 2},
\end{equation}
which is the conclusion of the lemma.
\end{proof}

\begin{theorem}[\cite{wainwright2019high}, Proposition 2.9]
\label{th:subexponential}
Suppose $X$ is sub-exponential with mean $\mu$ and parameters $(\nu, b)$.
Then
\begin{equation}
\pb{X \geq \mu + t} \leq \left\lbrace
\begin{array}{ll}
e^{-t^2/2\nu^2}, & \text{if } 0 \leq t \leq \nu^2/b, \\
e^{-t/2b}, & \text{if } t > \nu^2/b, \\
\end{array}
\right.
\end{equation}
\end{theorem}

\begin{proof}
This follows from the classical Chernoff bound.
\end{proof}

\subsection{Concentration bounds}

We prove this slightly more general lemma,
whose special case $\kappa = 8 \sqrt{\ln (1+d)}$ yields Proposition~\ref{prop:subspace}.

\begin{lemma}
\label{lemma:subspace}
    Let $u$ be a unit vector.
    Suppose $V$ is an isotropically random $d$-dimensional subspace of $\setR^D$.
    Then, for all $\kappa \geq 1$, 
    with high probability at least $1 - 3 d^2 e^{-\kappa^2 / 16}$, 
    for all unit vectors $v \in V$, $\absv{u^T v} \leq 6 \kappa \sqrt{\frac{d}{D}}$.
\end{lemma}

\begin{proof}
    Note that the proposition is trivial for $\kappa \geq \frac{1}{6} \sqrt{\frac{D}{d}}$,
    as the right-hand side is 1, 
    which is clearly an upper-bound on the scalar product of two unit vectors.
    Without loss of generality, in the sequel, 
    we assume $1 \leq \kappa \leq \frac{1}{6} \sqrt{\frac{D}{d}} < \sqrt{D}$.

    Define $V$ to be the vector spanned by $d$ independent vectors $w_1, \ldots, w_d \sim \calN(0, I_D)$.
    With probability 1, $V$ has dimension $d$.
    Plus it is clearly isotropically drawn.
    Moreover, $(\frac{w_1}{\sqrt{D}}, \ldots, \frac{w_d}{\sqrt{D}})$ is then a quasi-orthonormal basis of $V$.
    More precisely, define the following events:
    \begin{align}
        \event_1 &\triangleq \set{ \forall i \in [d] \mathsep \absv{ \norm{w_i}{2}^2 - D } \leq \kappa \sqrt{D} }, \\
        \event_2 &\triangleq \set{ \forall i \neq j \mathsep \absv{ w_i^T w_j } \leq \kappa \sqrt{D} }, \\
        \event_3 &\triangleq \set{ \forall i \in [d] \mathsep \absv{w_i^T u} \leq \kappa }.
    \end{align}
    By Lemma~\ref{lemma:chi2}, using $\kappa < \sqrt{D}$, 
    we know that $\event_1$ holds with probability at least $1-de^{-\kappa^2/8}$.
    By combining Lemma~\ref{lemma:product_normal_subexponential}, Lemma~\ref{lemma:addition_of_subexponentials},
    we know that $w_i^T w_j$ is zero-mean and sub-exponential of parameters $(2\sqrt{2D}, 4)$.
    Thus by Theorem~\ref{th:subexponential} (with $t = \kappa \sqrt{D} < D$),
    $\event_2$ holds with probability $1 - d^2 e^{-\kappa^2 / 16}$.
    Finally, $w_i^T u \sim \calN(0, u^T I u) = \calN(0, 1)$.
    By Lemma~\ref{lemma:normal_tail_bound}, 
    $\event_3$ holds with probability at least 
    $1 - \frac{d}{\kappa \sqrt{\tau}} e^{-\kappa^2 / 2} \geq 1 - d e^{-\kappa^2 / 2}$,
    using $\kappa \geq 1$.
    Taking the intersection $\event = \event_1 \cap \event_2 \cap \event_3$ then yields
    \begin{align}
        \pb{\event} &\geq 1 - d e^{-\kappa^2 / 8} - d^2 e^{-\kappa^2 / 16} - d e^{-\kappa^2 / 2} \\
        &\geq 1 - 3 d^2 e^{-\kappa^2 / 16}.
    \end{align}
    In the sequel, we work under this high-probability event.
    
    Now consider any (not necessarily unit) vector $z \in V$.
    Then there exists unique values $\lambda_1, \ldots, \lambda_d$ such that $z = \sum \lambda_i w_i$.
    But now, we have
    \begin{align}
        \absv{z^T u}
        &= \absv{\sum_{i=1}^d \lambda_i w_i^T u}
        \leq \sum_{i=1}^d \absv{\lambda_i} \absv{w_i^T u} 
        \leq \kappa \norm{\lambda}{1}
        \leq \kappa \sqrt{d} \norm{\lambda}{2},
    \end{align}
    where we used the classical inequality $\norm{\cdot}{1} \leq \sqrt{d} \norm{\cdot}{2}$ in $\setR^d$.
    Moreover, note that
    \begin{align}
        \norm{z}{2}^2
        &= \norm{\sum \lambda_i w_i}{2}^2 
        = \sum_i \lambda_i^2 \norm{w_i}{2}^2 
            + 2 \sum_{i \neq j} \lambda_i \lambda_j w_i^T w_j 
        \geq (D - \kappa \sqrt{D}) \norm{\lambda}{2}^2
            - 2 \kappa \sqrt{D} \sum_{i \neq j} \lambda_i \lambda_j \\
        &\geq (D - \kappa \sqrt{D}) \norm{\lambda}{2}^2
            - 2 \kappa \sqrt{D} \norm{\lambda}{1}^2 
        \geq (D - \kappa \sqrt{D}) \norm{\lambda}{2}^2
            - 2 \kappa \sqrt{dD} \norm{\lambda}{2}^2 \\
        &= \left( 1 - \frac{(1+2\sqrt{d}) \kappa}{\sqrt{D}} \right) D \norm{\lambda}{2}^2 
        \leq \left( 1 - 3\kappa \sqrt{\frac{d}{D}} \right) D \norm{\lambda}{2}^2 
        \leq \frac{1}{2} D \norm{\lambda}{2}^2,
    \end{align}
    where the last line uses $\kappa \leq \frac{1}{6} \sqrt{\frac{D}{d}}$.
    Now denoting $v = z / \norm{z}{2}$ yields
    \begin{align}
        \absv{v^T u}
        &\leq \frac{\absv{z^T u}}{\norm{z}{2}} 
        \leq \frac{\kappa \sqrt{d} \norm{\lambda}{2}}{\sqrt{\frac{1}{2} D \norm{\lambda}{2}^2}} 
        = \sqrt{2} \kappa \sqrt{\frac{d}{D}}
        \leq 6 \kappa \sqrt{\frac{d}{D}},
    \end{align}
    which concludes the proof.
\end{proof}

\begin{lemma}
\label{lemma:subspace2}
    Let $d \geq 2$ and $u$ be a unit vector.
    Suppose $V$ is an isotropically random $d$-dimensional subspace of $\setR^D$.
    Then, for all $\nu \geq 1$, 
    with high probability at least $1 - 23 e^{-4 \nu^2}$, 
    for all unit vectors $v \in V$, $\absv{u^T v} \leq 48 \nu \sqrt{\frac{d \ln (1+d)}{D}}$.
\end{lemma}

\begin{proof}
    Let $\nu \geq 1$.
    Define $\kappa \triangleq 8 \nu \sqrt{\ln (1+d)}$.
    Then $\kappa \geq 8 \sqrt{\ln (1+d)} \geq 1$, for $d \geq 2$.
    
    Thus Lemma~\ref{lemma:subspace} applies: 
    with probability at least $p \triangleq 1 - 3 d^2 e^{-\kappa^2 / 16}$,
    \begin{align}
        \forall~\text{unit}~v \in V \mathsep 
        \absv{u^T v} \leq 6 \kappa \sqrt{\frac{d}{D}}
        = 48 \nu \sqrt{\frac{d \ln(1+d)}{D}}.
    \end{align}
    Now using $1+d \geq 3 \geq e$ and $2-4\nu^2 \leq 2-4 = -2 < 0$, we have
    \begin{align}
        p 
        &= 1 - 3 d^2 e^{-4 \nu^2 \ln(1+d)} 
        = 1 - 3 d^2 (1+d)^{- 4 \nu^2} \\
        &\geq 1 - \frac{3 d^2}{(1+d)^2} (1+d)^{2-4\nu^2}
        \geq 1 - 3 e^{2-4\nu^2} \\
        &\geq 1 - (3 e^2) e^{-4\nu^2}
        \geq 1 - 23 e^{-4\nu^2},
    \end{align}    
    which concludes the proof.
\end{proof}

\subsection{Proof of Proposition~\ref{prop:subspace}}

Let us first prove a bound on an infinite series.

\begin{lemma}
\label{lemma:series3}
    $\sum_{k = 0}^\infty (1+k) e^{- 4 k^2} \leq 1.04$.
\end{lemma}

\begin{proof}
    Note that $\frac{d}{dt} t e^{-4 t^2} = (1 - 8t^2) e^{-4t^2} \leq 0$ for $t \geq 1$.
    Thus, for $k \geq 2$ and $t \in [k-1, k]$, we have $k e^{-4k^2} \leq t e^{-4t^2}$.
    Now consider $K \geq 2$.
    Then
    \begin{align}
        \sum_{k = 0}^\infty (1+k) e^{- 4 k^2}
        &= \sum_{k = 0}^{K-1} (1+k) e^{- 4 k^2}
        + \sum_{k = K}^\infty e^{- 4 k^2}
        + \sum_{k = K}^\infty k e^{- 4 k^2} \\
        &\leq \sum_{k = 0}^{K-1} (1+k) e^{- 4 k^2}
        + \sum_{k = K}^\infty e^{- 4 k}
        + \int_{K-1}^\infty t e^{- 4 t^2} dt \\
        &\leq \sum_{k = 0}^{K-1} (1+k) e^{- 4 k^2}
        + \frac{e^{-4K}}{1-e^{-4}}
        - \frac{1}{8} \left[ e^{- 4 t^2} \right]_{K-1}^\infty \\
        &\leq \sum_{k = 0}^{K-1} (1+k) e^{- 4 k^2}
        + \frac{e^{-4K}}{1-e^{-4}}
        + \frac{1}{8} e^{- 4 (K-1)^2}.
    \end{align}
    Using $K=2$ yields the bound.
\end{proof}

Finally we can conclude.

\begin{proof}[Proof of Proposition~\ref{prop:subspace}]
    Denote $\mu \triangleq \expect{v}$, and $u \triangleq \unit(\mu)$.
    Define $\event_0 = \emptyset$
    and $\event_\nu \triangleq \set{ \forall z \in V, \absv{u^T z} \leq 48 \kappa \norm{z}{2} \sqrt{\frac{d \ln(1+d)}{D}}}$ for $\nu \in \setN - \set{0}$.
    Like in Lemma~\ref{lemma:bound_rho_d},
    we bound $\norm{\mu}{2}$ by considering a partition into events $\event_\nu$
    and applying Lemma~\ref{lemma:subspace2}, yielding
    \begin{align}
        \norm{\mu}{2} 
        &= u^T \mu
        = \expect{ u^T v }
        = \sum_{\nu = 1}^\infty \expect{ u^T v | \event_\nu - \event_{\nu - 1} } \pb{\event_\nu - \event_{\nu - 1}} \\
        &\leq \sum_{\nu = 2}^\infty 48 \nu \sqrt{\frac{d \ln(1+d)}{D}} \pb{\neg \event_{\nu - 1}} \\
        &\leq 48 \sqrt{\frac{d \ln(1+d)}{D}} \sum_{\nu = 1}^\infty 23 \nu e^{- 4 (\nu-1)^2} \\
        &= 1104 \sqrt{\frac{d \ln(1+d)}{D}} \underbrace{\sum_{\nu = 1}^\infty \nu e^{- 4 (\nu-1)^2}}_{\leq 1.04},
    \end{align}
    using Lemma~\ref{lemma:series3}.
    The inequality $1104 \cdot 1.04 \leq 1150$ allows to conclude.
\end{proof}

\section{Experimental details}
\label{app:experiments}

The experiments were run on a machine with Nvidia GeForce GTX 970M with Cuda 13.
The code is provided in the supplementary material, and can be straightforwardly reproduced, 
even on a CPU,
by following the instructions from \texttt{README.md}.

In order to facilitate convergence, 
and since the optimization problem is strongly convex for linear models,
in the synthetic part of the experiments, we consider a dynamic learning rate,
where the learning rate is multiplied by an update term smaller than 1 (in our experiments, equal to $0.9$),
when, after 10 iterations, the norm of the estimated (robustified) gradient decreases.

The initial learning rate is fixed at $0.05$, and its value at iteration $t$ is given by
$u_t 0.05 / \sqrt{t}$,
where $u_t = 0.9^{k_t}$ with $k_t$ being the number of the itmes that the learning rate
gets a dynamic update because of a failure to decrease the norm of the gradient.

Our experiments all exhibit a convergence to a point,
for which the norm of the estimated (robustified) gradient is very small.

\end{document}